\documentclass[conference]{IEEEtran}
\IEEEoverridecommandlockouts

\usepackage{cite}
\usepackage{amsmath,amssymb,amsfonts,amsthm}
\usepackage{algorithmic}
\usepackage{graphicx}
\usepackage{textcomp}
\usepackage{xcolor,url}
\def\BibTeX{{\rm B\kern-.05em{\sc i\kern-.025em b}\kern-.08em
    T\kern-.1667em\lower.7ex\hbox{E}\kern-.125emX}}

\newtheorem{theorem}{Theorem}
\newtheorem{lemma}{Lemma}

\newtheorem{corollary}{Corollary}
\newtheorem{definition}{Definition}

\newcommand{\cT}{\mathcal{T}}
\newcommand{\bR}{\mathbb{R}}
\newcommand{\bE}{\mathbb{E}}

\newcommand{\eps}{\varepsilon}
\newcommand{\cH}{\mathcal{H}}
\newcommand{\cD}{\mathcal{D}}
\newcommand{\Reg}{\text{Reg}}
\newcommand{\cC}{\mathcal{C}}
\newcommand{\cI}{\mathcal{I}}
\newcommand{\LCB}{\text{LCB}}
\newcommand{\UCB}{\text{UCB}}
\newcommand{\cE}{\mathcal{E}}

\begin{document}

\title{Communication-Corruption Coupling and Verification in Cooperative Multi-Objective Bandits
\thanks{}
}

\author{\IEEEauthorblockN{Ming Shi}
\IEEEauthorblockA{\textit{Dept. of Electrical Engineering} \\
\textit{University at Buffalo}\\
Buffalo, NY, USA \\
mshi24@buffalo.edu}
}

\maketitle

\begin{abstract}
We study cooperative stochastic multi-armed bandits with vector-valued rewards under adversarial corruption and limited verification. In each of $T$ rounds, each of $N$ agents selects an arm, the environment generates a clean reward vector, and an adversary perturbs the observed feedback subject to a global corruption budget $\Gamma$. Performance is measured by team regret under a coordinate-wise nondecreasing, $L$-Lipschitz scalarization $\phi$, covering linear, Chebyshev, and smooth monotone utilities. Our main contribution is a communication-corruption coupling: we show that a fixed environment-side budget $\Gamma$ can translate into an effective corruption level ranging from $\Gamma$ to $N\Gamma$, depending on whether agents share raw samples, sufficient statistics, or only arm recommendations. We formalize this via a protocol-induced multiplicity functional and prove regret bounds parameterized by the resulting effective corruption. As corollaries, raw-sample sharing can suffer an $N$-fold larger additive corruption penalty, whereas summary sharing and recommendation-only sharing preserve an unamplified $O(\Gamma)$ term and achieve centralized-rate team regret. We further establish information-theoretic limits, including an unavoidable additive $\Omega(\Gamma)$ penalty and a high-corruption regime $\Gamma=\Theta(NT)$ where sublinear regret is impossible without clean information. Finally, we characterize how a global budget $\nu$ of verified observations restores learnability. That is, verification is necessary in the high-corruption regime, and sufficient once it crosses the identification threshold, with certified sharing enabling the team's regret to become independent of $\Gamma$.
\end{abstract}


\section{Introduction}\label{sec:intro}

Stochastic multi-armed bandits (MAB) provide a canonical model for sequential decision-making with unknown reward distributions and a regret objective \cite{lai1985,auer2002,bubeck2012}. A growing class of applications, however, is inherently distributed and multi-metric: multiple cooperative agents (devices, robots, base stations, edge servers) repeatedly probe the same set of actions while optimizing a vector of objectives such as latency, reliability, energy, and fairness. These settings naturally lead to \emph{multi-agent multi-objective bandits}, where learning speed is shaped not only by the stochasticity of the environment but also by \emph{how agents communicate and aggregate information} \cite{agarwal2022limitedcomm,madhushani2021imperfect}.

A central obstacle in such deployments is that feedback can be \emph{strategically corrupted}, e.g., measurements may be perturbed by sensor faults, compromised telemetry, click fraud, fake reviews, or data poisoning. This has motivated the study of \emph{stochastic bandits with adversarial corruptions}, where an adaptive adversary perturbs observed rewards subject to a global budget \cite{lykouris2018,gupta2019,ito2021}. Existing theory largely characterizes a single learner’s regret as a stochastic term plus corruption term tradeoff. In contrast, in cooperative systems the communication layer can \emph{replicate} corrupted samples across agents (e.g., via raw-sample broadcasting), thereby altering the statistical effect of the same environment-side corruption budget. This communication-induced amplification phenomenon is orthogonal to (and not captured by) standard models of cooperative bandits under clean feedback and limited/imperfect communication \cite{agarwal2022limitedcomm,madhushani2021imperfect}, and decentralized robustness to malicious/Byzantine agents \cite{zhu2023byzantine,shi2025power}. Moreover, corruption is especially consequential in \emph{multi-objective} learning. Each pull returns a reward vector, and small perturbations can distort implied trade-offs, invalidate dominance relations, and break optimism based on scalar utilities. While multi-objective bandits have been studied under clean feedback (e.g., Pareto or scalarization objectives) \cite{roijers2013survey,drugan2013,auer2016pareto,crepon2024pareto,cao2025provably}, their intersection with budgeted adversarial corruption and cooperative communication is not well understood.

We study cooperative $N$-agent stochastic multi-objective bandits with vector rewards in $[0,1]^d$. Agents exchange messages according to one of three canonical protocols: raw-sample sharing (S1), sufficient-statistic sharing (S2), or recommendation-only sharing (S3). An adaptive adversary corrupts unverified observations under a \emph{global} corruption budget $\Gamma$, and the team may additionally acquire at most $\nu$ \emph{verified} (clean) observations system-wide. Performance is measured by \emph{team regret} under a monotone $L$-Lipschitz scalarization $\phi$ (covering linear, Chebyshev/min, and smooth log-sum-exp utilities) \cite{roijers2013survey}. This setting leads to three information-theoretic questions that are specific to cooperative systems:
\begin{enumerate}
    \item \emph{Communication--corruption coupling:} how does the regret depend on $(\Gamma,N)$ under different sharing protocols?
    \item \emph{Protocol-induced phase transitions:} can a fixed environment-side budget $\Gamma$ yield qualitatively different learnability depending on whether corrupted data are replicated through communication?
    \item \emph{Clean side information:} how much verification budget $\nu$ is necessary and sufficient to recover learnability when corruption is severe?
\end{enumerate}

Our main contributions provide a protocol-level characterization of corruption robustness in cooperative multi-objective bandits.
\begin{itemize}
\item \textbf{A protocol-level corruption functional (effective corruption).}
We formalize how a cooperative algorithm uses each underlying agent-round observation across the network via a \emph{multiplicity} (replication) factor, which induces an \emph{effective corruption} level.
This yields a sharp distinction. That is, (S1) can inflate the corruption impact by up to a factor $N$ via replication, whereas (S2)-(S3) avoid replication and preserve the original $O(\Gamma)$ corruption penalty.

\item \textbf{A meta regret theorem parameterized by effective corruption.} We prove a protocol-agnostic robust-UCB guarantee in which the corruption term scales with the induced effective corruption (rather than $\Gamma$ itself), and we show how monotone scalarizations can be handled via a dominance-preserving vector-to-scalar optimism principle.

\item \textbf{Tight separation between sharing modes (S1) versus (S2) and (S3).} For (S1), we provide a matching lower bound showing that naive ``append-all'' raw sharing can be \emph{information-theoretically} $N$-times worse in the corruption term. For (S2) and (S3), we obtain centralized-rate team regret (up to logarithmic factors) with only an \emph{unamplified} $O(\Gamma)$ corruption penalty.

\item \textbf{Verification and certified sharing under high corruption.} To address regimes where corrupted feedback alone is insufficient, we introduce a verification channel and show that sharing \emph{verified certificates} enables the team to filter corrupted recommendations and regain identifiability with $\nu$ scaling on the order of $K\log(\cdot)/\Delta_{\min}^2$ (up to constants and Lipschitz factors), even when $\Gamma$ is linear in the total number of pulls. This mechanism is distinct from prior multi-agent robustness models (e.g., Byzantine agents) \cite{zhu2023byzantine} and from existing corruption-robust multi-agent bandits that do not treat multi-objective scalarization and verification jointly \cite{ghaffari2025multiagentcorrupt}.
\end{itemize}

Overall, our results show that in cooperative multi-objective bandits, \emph{what agents share} can be as consequential as \emph{how much corruption the environment injects}. Communication protocols reshape the effective statistical contamination, and limited verified information can be leveraged to restore learnability in the high-corruption regime.

\section{Problem Formulation}\label{sec:problem}

\emph{Notation:} Let $[K] \triangleq \{1,\dots,K\}$ denote the arm set and $d$ the number of objectives. For $x,y \in \bR^d$, write $x \preceq y$ for coordinate-wise inequality. Unless stated otherwise, $\|\cdot\| = \|\cdot\|_\infty$. Let $\mathbf{1} \in \bR^d$ be the all-ones vector. There are $N$ agents indexed by $n \in [N]$ and the horizon is $T$ rounds.

\subsection{Multi-Agent Stochastic Multi-Objective Bandits}\label{subsec:ma_env}

Each arm $k$ is associated with an unknown distribution $\cD_k$ supported on $[0,1]^d$, with mean vector $\mu_k \triangleq \bE_{R\sim\cD_k}[R]\in[0,1]^d$. At each round $t=1,2,\dots,T$, agent $n$ chooses an arm $k_{n,t}\in[K]$. Conditioned on joint action $\{k_{n,t}\}_{n=1}^N$, the environment draws clean rewards $R_{n,t}\sim \cD_{k_{n,t}}$, independently across agents and time (conditioned on the chosen arms). Let $\cH_{n,t-1}$ be the $\sigma$-field generated by agent $n$'s past observations and all messages received up to time $t-1$ under the communication model in Section~\ref{subsec:comm}. All decisions of agent $n$ at round $t$ are measurable with respect to (w.r.t.) $\cH_{n,t-1}$.

\subsection{Verification and Corrupted Observations}\label{subsec:ver_corr}

After choosing $k_{n,t}$ and before receiving reward feedback, each agent $n$ selects a verification decision $V_{n,t}\in\{0,1\}$ that is measurable w.r.t. $\cH_{n,t-1}$. If $V_{n,t}=1$, agent $n$ observes the clean reward $R_{n,t}$ (clean side information). We restrict the verification to be constrained by a global budget, i.e.,
\begin{align}\label{eq:global_verify}
\sum\nolimits_{t=1}^T \sum\nolimits_{n=1}^N V_{n,t}\le \nu.
\end{align}
If $V_{n,t}=0$, an adversary selects a corruption vector $C_{n,t}\in\bR^d$ and agent $n$ will then observe a corrupted reward
\begin{equation}\label{eq:corrupted_obs_ma}
\widetilde{R}_{n,t}=\Pi_{[0,1]^d}\big(R_{n,t}+C_{n,t}\big),
\end{equation}
where $\Pi_{[0,1]^d}$ is coordinate-wise projection. The adversary may be adaptive to the past public transcript (joint actions, all messages, and all previously revealed feedback), and may also depend on the current $(k_{n,t},V_{n,t})$ before choosing $C_{n,t}$. However, corruption is constrained by a global budget, i.e.,
\begin{align}\label{eq:global_budget}
\sum\nolimits_{t=1}^T \sum\nolimits_{n=1}^N \|C_{n,t}\|\le \Gamma.
\end{align}
Thus, we define the actually used feedback
\begin{align}
X_{n,t}\triangleq 
\begin{cases}
R_{n,t}, & V_{n,t}=1 \text{ (verified round)},\\
\widetilde{R}_{n,t}, & V_{n,t}=0 \text{ (unverified round)}.
\end{cases}
\end{align}

\subsection{Scalarization and Team Regret}\label{subsec:scalarization}

We consider multi-objective scalarizations\footnote{This class includes linear scalarizations $\phi(x)=w^\top x$ with $w\in\Delta_d$, Chebyshev scalarization $\phi(x)=\min_{i\in[d]}x_i$, and log-sum-exp smoothing $\phi(x)=\beta^{-1}\log\sum_{i=1}^d e^{\beta x_i}$.} $\phi:[0,1]^d\to\bR$ that are coordinate-wise nondecreasing (i.e., $x\preceq y\Rightarrow \phi(x)\le \phi(y)$) and $L$-Lipschitz under $\ell_\infty$ (i.e., $|\phi(x)-\phi(y)|\le L\|x-y\|_\infty$ for all vectors $x,y\in[0,1]^d$). Define $\theta_k\triangleq \phi(\mu_k)$ and let $k^*\in\arg\max_{k\in[K]}\theta_k$ be the best arm with largest scalarized mean reward. For $k\neq k^*$, define the gap $\Delta_k\triangleq \theta_{k^*}-\theta_k>0$.

We measure team scalarization regret as follows,
\begin{equation}\label{eq:team_regret}
\Reg_\phi^{\mathrm{team}}(T)
\triangleq \sum_{t=1}^T\sum_{n=1}^N\big(\theta_{k^*}-\theta_{k_{n,t}}\big)
= \sum_{k\neq k^*}\Delta_k\,N_k^{\mathrm{team}}(T),
\end{equation}
where $N_k^{\mathrm{team}}(T)\triangleq\sum_{t=1}^T\sum_{n=1}^N\mathbf{1}\{k_{n,t}=k\}$.

\subsection{Communication and Sharing Models}\label{subsec:comm}

At the end of each round $t$, all agents broadcast messages over a reliable channel. We consider three canonical message types as follows (and the algorithm additionally specifies how received information is aggregated).
\begin{itemize}
    \item \emph{(S1) Raw-sample sharing (append-all):} Agent $n$ broadcasts $(k_{n,t},X_{n,t},V_{n,t})$. Each agent maintains a local multiset of received triples and updates its per-arm estimates by appending all received samples as unit-weight data.
    \item \emph{(S2) Sufficient-statistic sharing (synchronized global aggregates):} Agent $n$ broadcasts per-arm cumulative statistics $\big(H_{n,k}(t),S_{n,k}(t),H^{\mathrm{ver}}_{n,k}(t),S^{\mathrm{ver}}_{n,k}(t)\big)_{k\in[K]}$, where $H_{n,k}(t)\triangleq\sum_{\tau\le t}\mathbf{1}\{k_{n,\tau}=k\}$ is the number of times agent $n$ has pulled arm $k$ up to round $t$, $S_{n,k}(t)\triangleq\sum_{\tau\le t: k_{n,\tau}=k,V_{n,\tau}=0}\widetilde R_{n,\tau}$ is the cumulative sum of the (possibly corrupted) observed reward vectors $\widetilde R_{n,\tau}$ obtained by agent $n$ from arm $k$ up to round $t$, $H^{\mathrm{ver}}_{n,k}(t)\triangleq\sum_{\tau\le t}\mathbf{1}\{k_{n,\tau}=k,V_{n,\tau}=1\}$ is the number of \emph{verified} pulls of arm $k$ by agent $n$ up to round $t$, $S^{\mathrm{ver}}_{n,k}(t)\triangleq\sum_{\tau\le t: k_{n,\tau}=k, V_{n,\tau}=1}R_{n,\tau}$ is the cumulative sum of the corresponding clean rewards $R_{n,\tau}$ from those verified pulls. Upon receiving all broadcasts, corresponding synchronized global aggregates are $H_k(t) = \sum_{n=1}^N H_{n,k}(t), S_k(t)=\sum_{n=1}^N S_{n,k}(t), H_k^{\mathrm{ver}}(t)=\sum_{n=1}^N H_{n,k}^{\mathrm{ver}}(t), S_k^{\mathrm{ver}}(t)=\sum_{n=1}^N S_{n,k}^{\mathrm{ver}}(t)$, and all agents can use the same global statistics to compute indices.
    \item \emph{(S3) Recommendation-only sharing:} Agent $n$ broadcasts only an arm index $M_{n,t}\in[K]$ (e.g., its local argmax index). The message could include a \emph{verified certificate} computed only from verified samples (defined in Section~\ref{subsec:verified_sharing}). No reward vectors are communicated.
\end{itemize}

\section{Main Results}\label{sec:mainresults}

We present our main theoretical results in this section. Please see the appendices and our technical report for the complete proofs~\cite{mingshihomepagetechrpt1}.

\subsection{Multiplicity and Effective Corruption}\label{subsec:effcorr}

A cooperative protocol specifies the message content (S1)-(S3) and an aggregation rule, i.e., which observed samples are incorporated (possibly multiple times) into the estimators that drive arm indices. We formalize this by tracking sample reuse.

Let $\widehat\mu_{j,k}(t)$ denote the empirical mean vector for arm $k$ used by \emph{estimator} $j$ at the end of round $t$. For each estimator $j$ and arm $k$, let $\cI_{j,k}(t)\subseteq [N]\times[t]$ be the multiset of agent-round indices $(n,\tau)$ whose (unverified) observations $\widetilde R_{n,\tau}$ are included with unit weight in $\widehat\mu_{j,k}(t)$.

\begin{definition}[Multiplicity and effective corruption]\label{def:rho_Geff}
For each agent-round $(n,t)$, we define its multiplicity
\begin{align}
\rho_{n,t}\triangleq \#\{ j:\ (n,t)\in \cI_{j,k_{n,t}}(T) \},
\end{align}
i.e., the number of distinct index-driving estimators that include $\widetilde R_{n,t}$ with unit weight. Define the effective corruption
\begin{align}\label{eq:Gamma_eff}
\Gamma_{\mathrm{eff}}\triangleq \sum\nolimits_{t=1}^T \sum\nolimits_{n=1}^N \rho_{n,t} \|C_{n,t}\|_\infty,
\end{align}
and the arm-wise effective corruption
\begin{align}\label{eq:Gamma_eff_k}
\Gamma_{\mathrm{eff},k}\triangleq \sum\nolimits_{t=1}^T \sum\nolimits_{n=1}^N \rho_{n,t}\,\|C_{n,t}\|_\infty\,\mathbf{1}\{k_{n,t}=k\}.
\end{align}
\end{definition}

\begin{lemma}[Effective corruption under (S1)-(S3)]\label{lem:protocol_amp_clean}
Under (S1) raw-sample sharing with append-all, each agent maintains its own local estimator, hence $\rho_{n,t}=N$ and $\Gamma_{\mathrm{eff}}=N\Gamma$.
Under (S2) synchronized sufficient-statistic sharing, all agents compute indices from a single synchronized global estimator, hence $\rho_{n,t}=1$ and $\Gamma_{\mathrm{eff}}=\Gamma$.
Under (S3) recommendation-only sharing, no raw samples are transmitted and each unverified observation is used only locally, hence $\rho_{n,t}=1$ and $\Gamma_{\mathrm{eff}}=\Gamma$.
\end{lemma}

Please see Appendix~\ref{proofoflem:protocol_amp_clean} for the complete proof of Lemma~\ref{lem:protocol_amp_clean}.

\subsection{A Concentration Lemma with Arm-Wise Effective Corruption}\label{subsec:conc}

The next lemma is the workhorse: any estimator whose arm-$k$ empirical mean is formed from $m$ unverified samples of arm $k$ incurs a stochastic fluctuation term $O(1/\sqrt{m})$ plus a bias term proportional to the total effective corruption mass assigned to arm $k$
divided by $m$.

\begin{lemma}[Vector mean concentration under effective corruption]\label{lem:conc_eff}
Fix $\delta\in(0,1)$.
With probability at least $1-\delta$, simultaneously for all estimators $j$, all arms $k$, and all times~$t$,
\begin{equation}\label{eq:conc_eff}
\big\|\widehat\mu_{j,k}(t)-\mu_k\big\|_\infty \le \sqrt{\frac{\log(2dKNT/\delta)}{2\,\max\{1,|\cI_{j,k}(t)|\}}} + \frac{\Gamma_{\mathrm{eff},k}}{\max\{1,|\cI_{j,k}(t)|\}}.
\end{equation}
\end{lemma}

Please see Appendix~\ref{proofofLemmalem:conc_eff} for the complete proof of Lemma~\ref{lem:conc_eff}.

\begin{proof}[Proof sketch]
Decompose $\widehat\mu_{j,k}(t)-\mu_k=(\overline\mu_{j,k}(t)-\mu_k)+(\widehat\mu_{j,k}(t)-\overline\mu_{j,k}(t))$, where $\overline\mu_{j,k}(t)$ is the mean of the corresponding clean rewards. The first term is bounded by coordinate-wise Hoeffding and union bound. For the second term, projection is nonexpansive in $\ell_\infty$, so $\|\widetilde R_{n,t}-R_{n,t}\|_\infty\le \|C_{n,t}\|_\infty$. Summing over the multiset $\cI_{j,k}(t)$ yields a bias bounded by the total corruption mass assigned to arm $k$ across all samples used by estimator $j$. Summing this worst-case usage over all estimators is precisely captured by $\Gamma_{\mathrm{eff},k}$.
\end{proof}

\subsection{Dominance Lifting for Monotone Scalarizations}\label{subsec:domlift}

\begin{definition}[Upper-closed confidence sets]\label{def:upperclosed}
A set $\cC\subseteq[0,1]^d$ is upper-closed if $x\in\cC$ and $x\preceq y\preceq \mathbf{1}$ implies $y\in\cC$.
\end{definition}

\begin{lemma}[Upper-corner reduction]\label{lem:corner_ma}
Let $\cC\subseteq[0,1]^d$ be upper-closed and define $x^{\max}\in[0,1]^d$ by $x^{\max}_j\triangleq \sup\{x_j: x\in\cC\}$. Then, $\sup_{x\in\cC}\phi(x)=\phi(x^{\max})$. In particular, for $\ell_\infty$ rectangles $\cC=\{x:\|x-\widehat\mu\|_\infty\le\beta\}\cap[0,1]^d$, we have
\begin{align}
\sup\nolimits_{x\in\cC}\phi(x)=\phi\big(\Pi_{[0,1]^d}(\widehat\mu+\beta\mathbf{1})\big).
\end{align}
\end{lemma}
Please see Appendix~\ref{proofofLemmalem:corner_ma} for the complete proof of Lemma~\ref{lem:corner_ma}.

\subsection{A Meta Regret Theorem Parameterized by $\Gamma_{\mathrm{eff}}$}\label{subsec:meta_ucb}

Denote the upper bound, i.e., the right-hand side, of (\ref{eq:conc_eff}) by $\beta_{j,k}(t)$. Consider any protocol that produces (possibly multiple) estimators $\widehat\mu_{j,k}(t)$ and uses the following index rule: each estimator $j$ forms an optimistic index $U_{j,k}(t)\triangleq \phi\Big(\Pi_{[0,1]^d}\big(\widehat\mu_{j,k}(t)+\beta_{j,k}(t)\mathbf{1}\big)\Big)$ with radius $\beta_{j,k}(t)$. Agents choose arms based on their designated estimator's indices (local in S1, global in S2, local in S3).

\begin{theorem}[Team regret bound via effective corruption]\label{thm:meta_team_ub}
Fix $\delta\in(0,1)$.
On an event of probability at least $1-\delta$, any cooperative UCB-type protocol using radii $\beta_{j,k}(t)$ satisfies
\begin{align}\label{eq:teamub_meta}
\Reg_\phi^{\mathrm{team}}(T) \le & c L\Big(\sqrt{KNT\log(2dKNT/\delta)} \nonumber \\
& + \Gamma_{\mathrm{eff}}\log(1+NT)\Big),
\end{align}
for a universal constant $c>0$.
\end{theorem}

Please see Appendix~\ref{proofofLemmathm:meta_team_ub} for the complete proof of Theorem~\ref{thm:meta_team_ub}.

\begin{proof}[Proof sketch]
On the concentration event from Lemma~\ref{lem:conc_eff}, the rectangle $\{x:\|x-\widehat\mu_{j,k}(t)\|_\infty\le \beta_{j,k}(t)\}$ contains $\mu_k$. Lemma~\ref{lem:corner_ma} implies optimism for $\theta_k=\phi(\mu_k)$. Whenever a suboptimal arm $k\neq k^*$ is selected by an estimator $j$, standard UCB reasoning yields $\Delta_k\le 2L\,\beta_{j,k}(t)$. Summing the resulting pull-count constraints across all estimators and arms produces a stochastic term $\tilde O(L\sqrt{KNT})$ and a corruption term proportional to $\sum_{n,t}\rho_{n,t}\|C_{n,t}\|_\infty = \Gamma_{\mathrm{eff}}$, since each corrupted sample can only ``pay for'' a bounded number of additional optimistic selections before confidence shrinks.
\end{proof}

\begin{corollary}[(S1) Raw-sample sharing incurs $N$-fold amplification]\label{cor:S1_amp}
Under (S1), $\Gamma_{\mathrm{eff}}=N\Gamma$ (Lemma~\ref{lem:protocol_amp_clean}), hence we have
\begin{align}
\Reg_\phi^{\mathrm{team}}(T)\le \widetilde O\big(L\sqrt{KNT}+LN\Gamma\big).
\end{align}
\end{corollary}

\begin{corollary}[(S2)--(S3) Achieve centralized corruption penalty]\label{cor:S2S3_rate}
Under (S2) or (S3), $\Gamma_{\mathrm{eff}}=\Gamma$, hence we have
\begin{align}
\Reg_\phi^{\mathrm{team}}(T)\le \widetilde O\big(L\sqrt{KNT}+L\Gamma\big).
\end{align}
\end{corollary}

\subsection{Lower Bound for Naive Raw Sharing}\label{subsec:naive_lb}

\begin{theorem}[Lower bound]\label{thm:amp_lb}
Under (S1) raw-sample sharing with append-all local estimators, there exist instances and adversaries with $\sum_{t,n}|C_{n,t}|\le \Gamma$ such that
\begin{align}
\bE\big[\Reg_\phi^{\mathrm{team}}(T)\big]\ge c_1\sqrt{KNT}+c_2 N\Gamma
\end{align}
for universal constants $c_1,c_2>0$.
\end{theorem}

Please see Appendix~\ref{proofofTheoremthm:amp_lb} for the complete proof of Theorem~\ref{thm:amp_lb}.

\begin{proof}[Proof sketch]
Construct a two-point testing instance where identifying the best arm requires resolving a mean gap $\Delta$. An adversary spends corruption mass $\Theta(\Gamma)$ on early samples of the informative arm. Under (S1), the same corrupted samples enter $N$ local estimators, reducing $N$ estimators' effective KL simultaneously. Le Cam's method yields an additional error probability bounded away from zero unless the protocol spends $\Omega(N\Gamma)$ regret mass to compensate.
\end{proof}

\subsection{Verification and Certified Recommendation Sharing}\label{subsec:verified_sharing}

Verification provides a clean channel immune to corruption. The key multi-agent issue is how to convert scattered verified reward samples into \emph{network-wide reliable decisions} under limited communication among all agnets. For each arm $k\in[K]$, we define the verified-only empirical mean ($\max\{1,H_k^{\mathrm{ver}}(t)\}$ is used to avoid division by zero when no verification has occurred yet)
\begin{align}
\widehat\mu_k^{\mathrm{ver}}(t)\triangleq \frac{1}{\max\{1,H_k^{\mathrm{ver}}(t)\}} \sum_{\tau\le t}\sum_{n:k_{n,\tau}=k, V_{n,\tau}=1} R_{n,\tau},
\end{align}
which averages only the clean reward vectors $R_{n,\tau}$ collected from verified pulls of arm $k$ up to time $t$.

\begin{lemma}[Verified scalar certificate]\label{lem:cert}
Fix $\delta\in(0,1)$. With probability at least $1-\delta$, for all $k$ and all $t$, we have
\begin{align}
\big|\phi(\widehat\mu_k^{\mathrm{ver}}(t))-\theta_k\big| \le L\sqrt{\frac{\log(2dKNT/\delta)}{2\,\max\{1,H_k^{\mathrm{ver}}(t)\}}}.
\end{align}
\end{lemma}

Please see Appendix~\ref{proofofLemmalem:cert} for the complete proof of Lemma~\ref{lem:cert}.

\begin{figure*}[t]
\centering
\begin{minipage}{.32\textwidth}
    \includegraphics[width=0.98\linewidth]{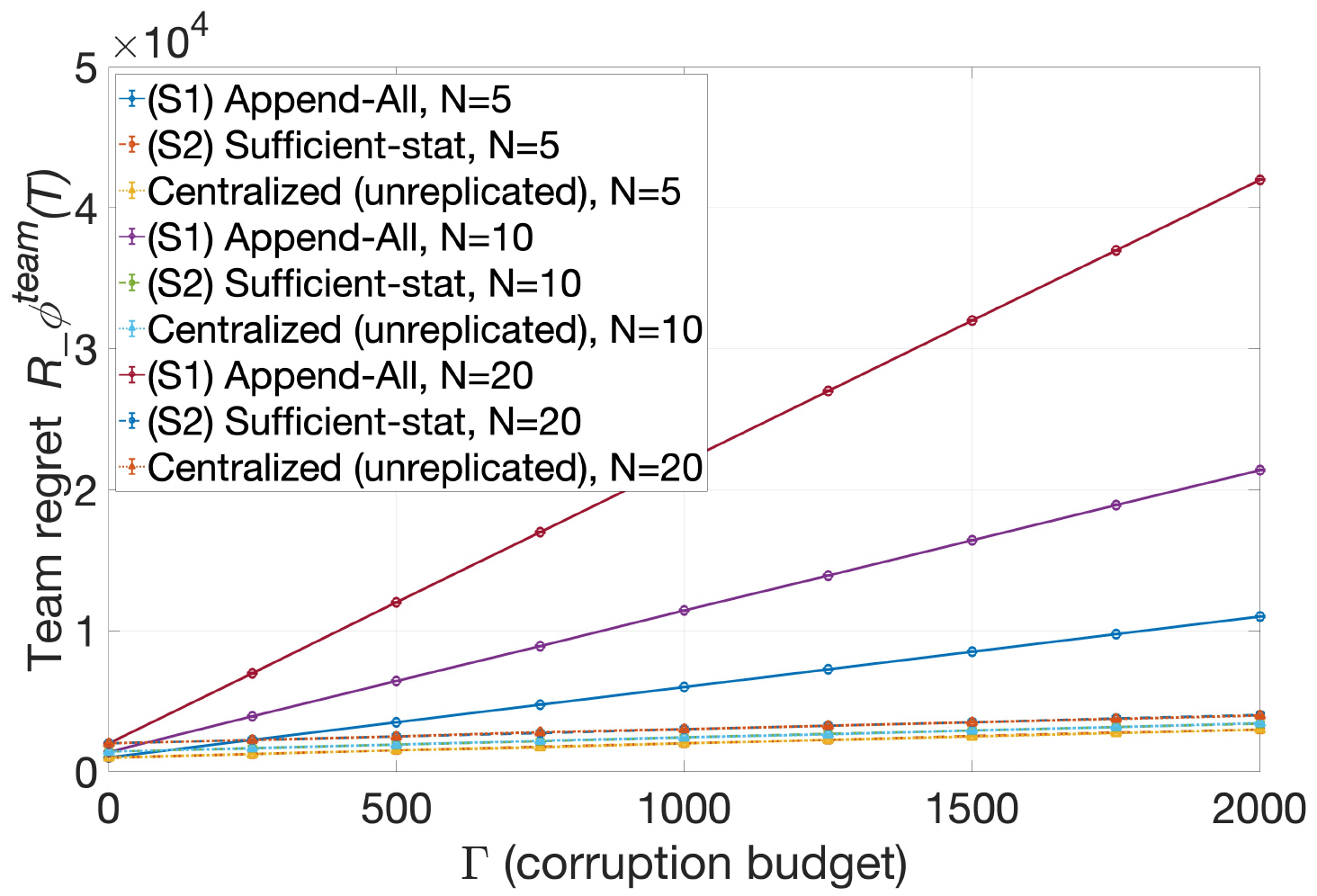}
    \caption{Team regret versus corruption budget $\Gamma$ for (S1) and (S2): (S1) scales like $\widetilde O(\sqrt{KNT}+N\Gamma)$ while (S2) scales like $\widetilde O(\sqrt{KNT}+\Gamma)$.}
    \label{fig:amp_vs_gamma}
\end{minipage}
\;
\begin{minipage}{.32\textwidth}
    \includegraphics[width=0.98\linewidth]{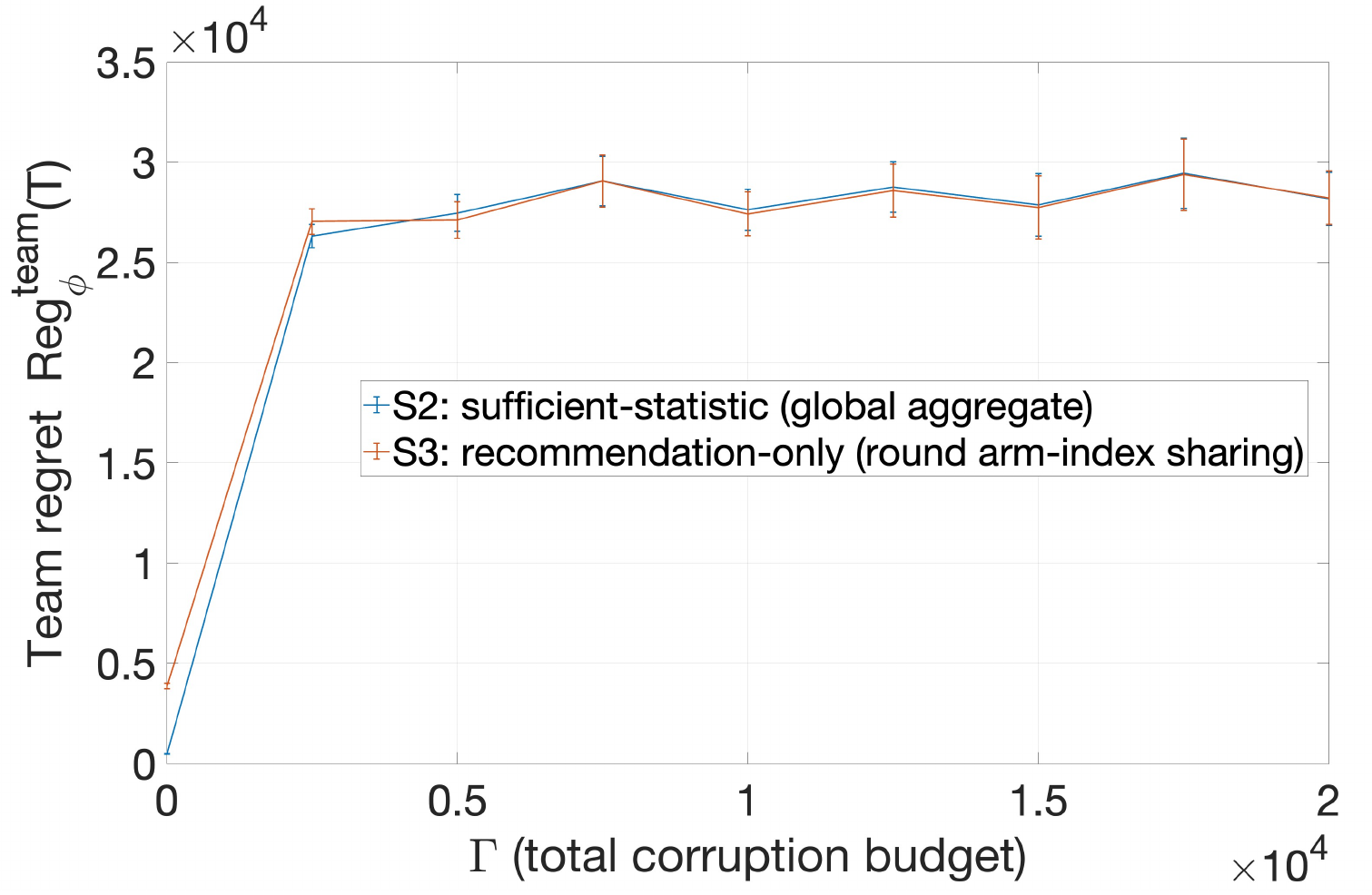}
    \caption{Team regret versus $\Gamma$ under (S2) and (S3): (S3) preserves the unamplified $O(\Gamma)$ term while trading off a clean-case coordination overhead.}
    \label{fig:s3_tradeoff}
\end{minipage}
\;
\begin{minipage}{.32\textwidth}
    \includegraphics[width=0.98\linewidth]{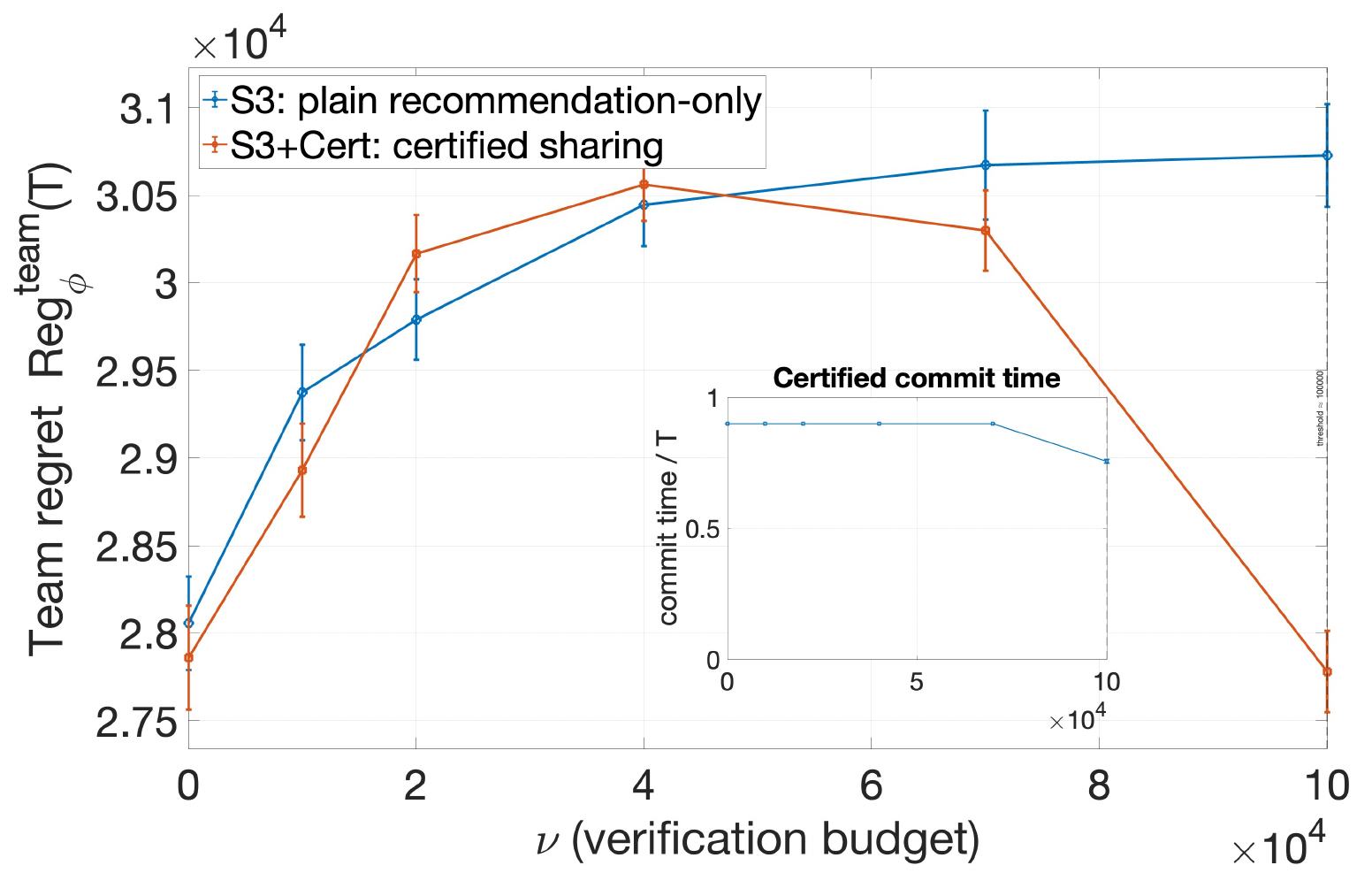}
    \caption{Team regret versus verification budget $\nu$ under (S3) with/without certified sharing in high-corruption regime $\Gamma=\Theta(NT)$: sharp improvement once $\nu \gtrsim K L^2 \Delta_{\min}^{-2}\log(dKNT)$.}
    \label{fig:verify_phase}
\end{minipage}
\end{figure*}

Under the recommendation-only communication model (S3), agents do not transmit reward vectors (which could replicate corrupted samples). However, pure recommendations $M_{n,t}$ can still be misleading under corruption. To make recommendations \emph{verifiable}, we allow an agent to optionally attach a certificate computed \emph{only from verified (clean) data}.

Concretely, at the end of round $t$, agent $n$ broadcasts a pair $(M_{n,t},\mathrm{cert}_{n,t})$ for $M_{n,t}\in[K]$, where $M_{n,t}$ is the arm it recommends (e.g., its current optimistic maximizer or most-played arm in a round). The certificate $\mathrm{cert}_{n,t}=\big(\LCB_{n,t},\UCB_{n,t}\big)$ is a scalar confidence interval for the \emph{true} scalarized value $\theta_{M_{n,t}}=\phi(\mu_{M_{n,t}})$ of the recommended arm, based solely on verified-only statistics. Specifically, we set $\LCB_{n,t}=\phi\!\big(\widehat\mu^{\mathrm{ver}}_{M_{n,t}}(t)\big)-\eps_{M_{n,t}}(t)$, $\UCB_{n,t}=\phi\!\big(\widehat\mu^{\mathrm{ver}}_{M_{n,t}}(t)\big)+\eps_{M_{n,t}}(t)$, where the half-width $\eps_k(t)\triangleq L\sqrt{\frac{\log(2dKNT/\delta)}{2\max\{1,H_k^{\mathrm{ver}}(t)\}}}$ follows from Hoeffding concentration applied coordinate-wise to verified rewards in $[0,1]^d$ (with a union bound over $d$ objectives, $K$ arms, and $NT$ agent-rounds) together with $L$-Lipschitzness of $\phi$ under $\ell_\infty$. Because $\mathrm{cert}_{n,t}$ depends only on verified samples, it is immune to arbitrary corruption on unverified rounds. Thus, even if an adversary can manipulate the unverified feedback that led agent $n$ to propose $M_{n,t}$, the interval $[\LCB_{n,t},\UCB_{n,t}]$ remains a valid high-probability bracket for $\theta_{M_{n,t}}$. Recipients can therefore \emph{filter} incoming recommendations by keeping only those with sufficiently tight certificates (large $H_k^{\mathrm{ver}}(t)$) and competitive lower bounds relative to others (formalized in the filtering rule below), ensuring that network-wide coordination is driven by clean evidence rather than corrupted impressions.

\begin{theorem}[High-corruption learnability via certified sharing]\label{thm:verified_sharing}
There exists a cooperative algorithm using (S3) recommendation-only sharing with optional verified certificates
and at most $\nu$ verifications such that:

\begin{enumerate}
\item (\emph{Always-valid robust baseline}) For all instances and adversaries with budget~\eqref{eq:global_budget},
\begin{align}
\Reg_\phi^{\mathrm{team}}(T)\le \widetilde O\big(L\cdot \Reg^{\mathrm{clean}}_{\mathrm{comm}}(K,N,T)+L\Gamma\big),
\end{align}
where $\Reg^{\mathrm{clean}}_{\mathrm{comm}}(K,N,T)$ is the regret overhead of the same coordination protocol in the clean case.

\item (\emph{Verification-driven override}) If $\Delta_{\min}\triangleq\min_{k\neq k^*}\Delta_k>0$ and
\begin{equation}\label{eq:cert_thresh}
\nu \ge c\,K\frac{L^2}{\Delta_{\min}^2}\log\!\Big(\frac{2dKNT}{\delta}\Big),
\end{equation}
then, with probability at least $1-\delta$, there exists $t^*\le T$ after which all agents play $k^*$ and
\begin{align}
& \Reg_\phi^{\mathrm{team}}(T) \le O(\nu\Delta_{\max}) \nonumber \\
& +\widetilde O\big(L\cdot \Reg^{\mathrm{clean}}_{\mathrm{comm}}(K,N,T)\wedge NT\Delta_{\max}\big),
\end{align}
independently of $\Gamma$ (even when $\Gamma=\Theta(NT)$).
\end{enumerate}
\end{theorem}

Theorem~\ref{thm:verified_sharing} gives a two-regime guarantee. In general, recommendation-only sharing avoids sample replication, so corruption contributes only an additive $O(\Gamma)$ term at the team level, plus $\Reg^{\mathrm{clean}}_{\mathrm{comm}}(K,N,T)$. Moreover, once $\nu \gtrsim \tilde{\Theta}(KL^2/\Delta_{\min}^2)$, verified certificates become tight enough to rule out suboptimal arms, since they rely only on clean samples, the adversary cannot forge them. Thus, the team can filter misleading recommendations and commit to $k^*$, making regret independent of $\Gamma$ even when $\Gamma=\Theta(NT)$. Please see Appendix~\ref{proofofTheoremthm:verified_sharing} for the complete proof of Theorem~\ref{thm:verified_sharing}.

\begin{proof}[Proof sketch (certificate dominance mechanism)]
Under \eqref{eq:cert_thresh}, the verified confidence width satisfies $\eps_k(t)\le \Delta_{\min}/4$ once $H_k^{\mathrm{ver}}(t)$ crosses threshold. By Lemma~\ref{lem:cert}, any certified suboptimal arm $k\neq k^*$ then satisfies $\UCB_k^{\mathrm{ver}}(t)\le \theta_k+\Delta_{\min}/4 < \theta_{k^*}-\Delta_{\min}/4 \le \LCB_{k^*}^{\mathrm{ver}}(t)$, so certified filtering retains only $k^*$. Because certificates depend solely on verified samples, corruption on unverified rounds cannot invalidate them. Broadcast ensures that once a valid $k^*$ certificate appears, all agents receive it and commit.
\end{proof}

\section{Numerical Results}\label{sec:experiments}

We empirically validate the communication-corruption-verification theory in Section~\ref{sec:mainresults}. Our experiments target four qualitative predictions, corruption amplification under raw-sample sharing (S1), centralized-rate robustness under sufficient-statistic sharing (S2), unamplified $O(\Gamma)$ robustness under recommendation-only sharing (S3) with reduced communication, and
high-corruption recovery under (S3) via certified sharing with verification budget $\nu$.

We simulate by setting $K=20$, $d=5$, $N\in\{5,10,20\}$, and $T=10^4$. Rewards are coordinate-wise Bernoulli: for each arm $k$, $R_{n,t}\sim \mathrm{Bernoulli}(\mu_k)$ independently across $n,t$ and coordinates, with $\mu_k\in[0,1]^d$. We test three monotone $L$-Lipschitz scalarizations: linear $\phi(x)=w^\top x$ ($w\in\Delta_d$), Chebyshev $\phi(x)=\min_j x^{(j)}$, and log-sum-exp $\phi(x)=\beta^{-1}\log\sum_j e^{\beta x^{(j)}}$. We compare all the three sharing modes (S1), (S2), and (S3). We report team regret $\Reg_\phi^{\mathrm{team}}(T)$, and average results over $50$ i.i.d.\ instances.

Figure~\ref{fig:amp_vs_gamma} plots $\Reg_\phi^{\mathrm{team}}(T)$ versus $\Gamma$ for (S1) and (S2). Consistent with Lemma~\ref{lem:protocol_amp_clean} and Theorem~\ref{thm:meta_team_ub}, (S1) exhibits an $N$-fold degradation in the corruption-dominated regime. The slope of regret versus $\Gamma$ scales approximately linearly with $N$. In contrast, (S2) closely tracks the centralized (unreplicated) benchmark and remains stable as $N$ increases.

Figure~\ref{fig:s3_tradeoff} compares (S3) recommendation-only sharing with (S2). As predicted by Corollary~\ref{cor:S2S3_rate} and Theorem~\ref{thm:verified_sharing}, (S3) avoids corruption amplification (the $\Gamma$-slope matches (S2)), but incurs an additional clean-case coordination overhead that is visible when $\Gamma$ is small. Communication is reduced from transmitting reward vectors or per-arm $d$-dimensional summaries to transmitting $O(N\log T)$ arm indices.

We set $\Gamma=\Theta(NT)$ and vary the verification budget~$\nu$, and Figure~\ref{fig:verify_phase} shows that plain (S3) recommendations can be misled by corrupted local estimates, yielding large regret even with moderate $\nu$. In contrast, certified sharing produces a sharp improvement once $\nu$ crosses the identification threshold in~\eqref{eq:cert_thresh}. Regret drops rapidly and the team commits to $k^{*}$ shortly thereafter, consistent with Theorem~\ref{thm:verified_sharing}.

\section{Conclusion}\label{sec:conclusion}

We studied cooperative $N$-agent stochastic multi-objective bandits under a global corruption budget $\Gamma$ and a verification budget $\nu$. Our main message is that robustness is jointly governed by \emph{communication} and \emph{corruption}: protocol-induced replication converts $\Gamma$ into an effective budget $\Gamma_{\mathrm{eff}}\in[\Gamma,\,N\Gamma]$, producing an $N$-fold gap between raw-sample sharing and summary/recommendation sharing. We proved a protocol-agnostic regret bound parameterized by $\Gamma_{\mathrm{eff}}$ for general monotone $L$-Lipschitz scalarizations, yielding tight corollaries for (S1)--(S3), and matching lower bounds showing the amplification under naive raw sharing is unavoidable. Finally, we identified a high-corruption regime $\Gamma=\Theta(NT)$ where sublinear team regret is impossible without clean information, and showed that verification restores learnability when shared as certified evidence: once $\nu$ exceeds the identification threshold, certificates enable reliable filtering and team-wide commitment with regret independent of $\Gamma$.

\bibliographystyle{ieeetr}
\bibliography{reference}

\newpage

\appendices

\section{Proof of Lemma \ref{lem:protocol_amp_clean}}\label{proofoflem:protocol_amp_clean}

\begin{proof}
We first make the multiplicity $\rho_{n,t}$ precise and then specialize it to the three protocols.

\subsection{Step 0: Formal Definition of Multiplicity}

Fix a cooperative algorithm $\mathsf{Alg}$ and fix a time $t\in[T]$. For each agent-round sample $(n,\tau)$ with $\tau\le t$, define its (protocol-induced) multiplicity
\begin{equation}\label{eq:def-rho}
\rho_{n,\tau}(t)
~\triangleq~
\sum_{n'=1}^N w_{n',t}(n,\tau),
\end{equation}
where $w_{n',t}(n,\tau)\ge 0$ is the \emph{weight} with which agent $n'$'s estimator at time $t$
uses the single observed vector $\widetilde R_{n,\tau}$ (or $R_{n,\tau}$ if verified) when forming its
arm-wise sufficient statistics (e.g., sums and counts) that determine its indices at time $t+1$.
In other words, for each arm $k$ and agent $n'$, the statistic used by $n'$ at time $t$ can be written as
\begin{align}
& S^{(n')}_{k}(t)
=
\sum_{\tau\le t}\sum_{n=1}^N w_{n',t}(n,\tau)\,\widetilde R_{n,\tau}\,\mathbf{1}\{k_{n,\tau}=k,V_{n,\tau}=0\}
\nonumber \\
& +
\sum_{\tau\le t}\sum_{n=1}^N w^{\mathrm{ver}}_{n',t}(n,\tau)\,R_{n,\tau}\,\mathbf{1}\{k_{n,\tau}=k,\,V_{n,\tau}=1\},
\end{align}
with the corresponding counts defined analogously by replacing $\widetilde R_{n,\tau}$ (resp.\ $R_{n,\tau}$) with $1$.
The weights $\{w_{n',t}(n,\tau)\}$ are completely determined by the communication/sharing rule and the deterministic
updates of $\mathsf{Alg}$; they do not depend on the reward distribution.
For the present lemma, we only need the following facts:
(i) under append-all updates each received raw sample is appended with unit weight;
(ii) under synchronized sufficient-statistic sharing there is a single global sum/count in which each raw sample appears once;
(iii) under recommendation-only sharing raw samples are never transmitted, hence a sample can only be used by its owner.

For conciseness, we will write $\rho_{n,\tau}$ instead of $\rho_{n,\tau}(T)$, i.e., the total multiplicity over the entire horizon,
since the proof below is pointwise in $t$ and hence also holds at $t=T$.

Given multiplicities $\{\rho_{n,t}\}$, the (protocol-dependent) effective corruption is
\begin{equation}\label{eq:Gamma-eff-proof}
\Gamma_{\mathrm{eff}}
~\triangleq~
\sum_{t=1}^T\sum_{n=1}^N \rho_{n,t}\,\|C_{n,t}\|_\infty.
\end{equation}
The environment-side budget is
\begin{equation}\label{eq:Gamma-budget-proof}
\Gamma
~\triangleq~
\sum_{t=1}^T\sum_{n=1}^N \|C_{n,t}\|_\infty.
\end{equation}

\subsection{Step 1: (S1) Raw-Sample Sharing with Append-All Implies $\rho_{n,t}=N$}

Under (S1), at the end of each round $t$, agent $n$ broadcasts the tuple $(k_{n,t}, \widetilde R_{n,t}, V_{n,t})$
(or $(k_{n,t},R_{n,t},V_{n,t})$ if verified). Under the append-all update rule, every recipient agent $n'\in[N]$
\emph{appends} each received raw sample into its local dataset as if it were an additional fresh sample with unit weight.
Therefore, for every agent-round pair $(n,t)$ and every recipient $n'\in[N]$, $w_{n',T}(n,t)=1$, (and similarly $w^{\mathrm{ver}}_{n',T}(n,t)=1$ on verified rounds). Plugging into \eqref{eq:def-rho} yields
\[
\rho_{n,t}
~=~
\sum_{n'=1}^N w_{n',T}(n,t)
~=~
\sum_{n'=1}^N 1
~=~
N.
\]
Substituting into \eqref{eq:Gamma-eff-proof} and using \eqref{eq:Gamma-budget-proof} gives
\[
\Gamma_{\mathrm{eff}}
~=~
\sum_{t=1}^T\sum_{n=1}^N N\,\|C_{n,t}\|_\infty
~=~
N\sum_{t=1}^T\sum_{n=1}^N \|C_{n,t}\|_\infty
~=~
N\Gamma.
\]

\subsection{Step 2: (S2) Synchronized Sufficient-Statistic Sharing Implies $\rho_{n,t}=1$}

Under (S2), each agent $n$ broadcasts the per-arm cumulative statistics
$\big(H_{n,k}(t), S_{n,k}(t), H^{\mathrm{ver}}_{n,k}(t), S^{\mathrm{ver}}_{n,k}(t)\big)_{k\in[K]}$.
The protocol specifies that all agents compute indices using the \emph{same} synchronized global aggregates
$H_k(t)=\sum_{n=1}^N H_{n,k}(t), S_k(t)=\sum_{n=1}^N S_{n,k}(t), H^{\mathrm{ver}}_k(t)=\sum_{n=1}^N H^{\mathrm{ver}}_{n,k}(t), S^{\mathrm{ver}}_k(t)=\sum_{n=1}^N S^{\mathrm{ver}}_{n,k}(t).$
Crucially, by the definitions of $H_{n,k}(t)$ and $S_{n,k}(t)$ in the formulation,
each raw observation $\widetilde R_{n,\tau}$ contributes \emph{exactly once} to exactly one summand in $S_{k_{n,\tau}}(t)$,
namely to $S_{n,k_{n,\tau}}(t)$, and hence exactly once to the global sum $S_{k_{n,\tau}}(t)$.
There is no mechanism that duplicates $\widetilde R_{n,\tau}$ across multiple $S_{n',k}$'s, because agents transmit only
already-aggregated sums/counts and the global aggregate is formed by addition, not by appending all raw samples into each local dataset.

Formally, for each $(n,\tau)$, the global sum $S_{k}(t)$ (and similarly $S_k^{\mathrm{ver}}(t)$) can be written as
\[
S_k(t)
~=~
\sum_{\tau'\le t}\sum_{n'=1}^N \widetilde R_{n',\tau'}\,\mathbf{1}\{k_{n',\tau'}=k,\,V_{n',\tau'}=0\},
\]
where the coefficient of $\widetilde R_{n,\tau}$ is exactly $1$ if $(k_{n,\tau}=k, V_{n,\tau}=0)$ and $0$ otherwise.
Thus, in the weight notation, for the estimator actually used to form indices, we have $w_{n',t}(n,\tau)=\mathbf{1}\{n'=n^*\}$, for some single (conceptual) synchronized estimator $n^*$, or equivalently (and more simply): the aggregate uses each sample with unit weight exactly once.
Hence $\rho_{n,\tau}=1$ for all $(n,\tau)$.
Substituting into \eqref{eq:Gamma-eff-proof} yields
\[
\Gamma_{\mathrm{eff}}
~=~
\sum_{t=1}^T\sum_{n=1}^N 1\cdot \|C_{n,t}\|_\infty
~=~
\Gamma.
\]

\subsection{Step 3: (S3) Recommendation-Only Sharing Implies $\rho_{n,t}=1$}

Under (S3), agents broadcast only an arm index $M_n(t)\in[K]$ (optionally with a certificate based on verified-only statistics),
and \emph{no reward vectors} $\widetilde R_{n,t}$ (or $R_{n,t}$) are ever transmitted.
Therefore, the only agent that can possibly include $\widetilde R_{n,t}$ in any estimator is the originating agent $n$ itself;
other agents do not observe $\widetilde R_{n,t}$ and cannot reconstruct it from $M_n(t)$.
Consequently, for every $(n,t)$ and for every other agent $n'\neq n$, $w_{n',T}(n,t)=0$, while $w_{n,T}(n,t)=1$ because agent $n$ uses its own sample once in its own local sufficient statistics.
Hence, \eqref{eq:def-rho} gives
\[
\rho_{n,t}
~=~
\sum_{n'=1}^N w_{n',T}(n,t)
~=~
1.
\]
Substituting into \eqref{eq:Gamma-eff-proof} yields $\Gamma_{\mathrm{eff}}=\Gamma$.

Combining Steps 1--3 proves the stated values of $\rho_{n,t}$ and $\Gamma_{\mathrm{eff}}$ for (S1)--(S3).
\end{proof}

\section{Proof of Lemma~\ref{lem:conc_eff}}\label{proofofLemmalem:conc_eff}

\begin{proof}
We prove a coordinate-wise bound and then take a union bound over coordinates, arms, times, and estimators.

\subsection{Step 0: Estimator/Sample-Set Notation}

Fix an estimator index $j$, an arm $k$, and a time $t$.
Let $\cI_{j,k}(t)$ denote the (multi-)set of agent-round indices whose observations are used with unit weight
by estimator $j$ to form the empirical mean of arm $k$ up to time $t$.\footnote{If the protocol duplicates samples (e.g., append-all under (S1)),
then $\cI_{j,k}(t)$ is naturally viewed as a multiset that may contain multiple copies of the same agent-round observation.}
Let $m \triangleq |\cI_{j,k}(t)|$. If $m=0$, then $\widehat\mu_{j,k}(t)$ can be defined arbitrarily (e.g., $0$); since $\|\widehat\mu_{j,k}(t)-\mu_k\|_\infty\le 1$
and the right-hand side of \eqref{eq:conc_eff} is at least $\sqrt{\log(2dKNT/\delta)/2}$, the bound holds trivially.
Henceforth assume $m\ge 1$.

Index the elements of $\cI_{j,k}(t)$ as $\{(n_r,\tau_r)\}_{r=1}^m$ (including repetitions if any).
For each $r\in[m]$, the underlying clean reward is $R_{n_r,\tau_r}\in[0,1]^d$ and the (possibly corrupted) observed reward is
\[
\widetilde R_{n_r,\tau_r}
~=~
\begin{cases}
R_{n_r,\tau_r}, & \text{if } V_{n_r,\tau_r}=1,\\
\Pi_{[0,1]^d}\!\big(R_{n_r,\tau_r}+C_{n_r,\tau_r}\big), & \text{if } V_{n_r,\tau_r}=0.
\end{cases}
\]
Then, the empirical mean used by estimator $j$ for arm $k$ at time $t$ is
\[
\widehat\mu_{j,k}(t)
~=~
\frac{1}{m}\sum_{r=1}^m \widetilde R_{n_r,\tau_r}.
\]

\subsection{Step 1: Decomposition Into Stochastic Fluctuation and Corruption Bias}

Using $\mu_k=\bE[R\mid k]$ and adding/subtracting clean rewards,
\begin{align}\label{eq:decomp}
\widehat\mu_{j,k}(t)-\mu_k
~=~
\underbrace{\frac{1}{m}\sum_{r=1}^m \big(R_{n_r,\tau_r}-\mu_k\big)}_{\text{stochastic term}} \nonumber \\
+
\underbrace{\frac{1}{m}\sum_{r=1}^m \big(\widetilde R_{n_r,\tau_r}-R_{n_r,\tau_r}\big)}_{\text{corruption term}}.
\end{align}
Taking $\|\cdot\|_\infty$ and using the triangle inequality gives
\begin{align}\label{eq:tri}
\big\|\widehat\mu_{j,k}(t)-\mu_k\big\|_\infty
\le
\left\|\frac{1}{m}\sum_{r=1}^m (R_{n_r,\tau_r}-\mu_k)\right\|_\infty
\nonumber \\
+
\left\|\frac{1}{m}\sum_{r=1}^m (\widetilde R_{n_r,\tau_r}-R_{n_r,\tau_r})\right\|_\infty.
\end{align}

\subsection{Step 2: Bounding the Corruption Term by $\Gamma_{\mathrm{eff},k}/m$}

We claim that for every $(n,\tau)$,
\begin{equation}\label{eq:proj_lip}
\|\widetilde R_{n,\tau}-R_{n,\tau}\|_\infty \;\le\; \|C_{n,\tau}\|_\infty,
\end{equation}
where we adopt the convention $C_{n,\tau}\equiv 0$ on verified rounds ($V_{n,\tau}=1$), so that \eqref{eq:proj_lip} holds uniformly.
Indeed, coordinate-wise projection onto $[0,1]$ is $1$-Lipschitz and fixes $R_{n,\tau}^{(\ell)}\in[0,1]$, hence
\begin{align}
\big|\widetilde R_{n,\tau}^{(\ell)}-R_{n,\tau}^{(\ell)}\big|
=
\big|\Pi_{[0,1]}\!\big(R_{n,\tau}^{(\ell)}+C_{n,\tau}^{(\ell)}\big)-\Pi_{[0,1]}\!\big(R_{n,\tau}^{(\ell)}\big)\big| \nonumber \\
\le |C_{n,\tau}^{(\ell)}|
\le \|C_{n,\tau}\|_\infty,
\end{align}
and taking the maximum over $\ell\in[d]$ yields \eqref{eq:proj_lip}.

Therefore,
\begin{align}\label{eq:corr_term_bound_1}
\left\|\frac{1}{m}\sum_{r=1}^m (\widetilde R_{n_r,\tau_r}-R_{n_r,\tau_r})\right\|_\infty
& \le
\frac{1}{m}\sum_{r=1}^m \|\widetilde R_{n_r,\tau_r}-R_{n_r,\tau_r}\|_\infty \nonumber \\
& \le
\frac{1}{m}\sum_{r=1}^m \|C_{n_r,\tau_r}\|_\infty.
\end{align}

We now relate the right-hand side to $\Gamma_{\mathrm{eff},k}$.
Let $w_{j}(n,\tau;t)\in\{0,1,2,\dots\}$ denote the number of times the agent-round observation $(n,\tau)$ appears in the multiset $\cI_{j,k}(t)$.
Then $m = \sum_{\tau'\le t}\sum_{n'=1}^N w_j(n',\tau';t)\,\mathbf{1}\{k_{n',\tau'}=k\}$, and
\begin{align}
& \sum_{r=1}^m \|C_{n_r,\tau_r}\|_\infty
\nonumber \\
& =
\sum_{\tau'\le t}\sum_{n'=1}^N w_j(n',\tau';t)\,\|C_{n',\tau'}\|_\infty\,\mathbf{1}\{k_{n',\tau'}=k\}.
\end{align}
By definition of the protocol-induced multiplicity $\rho_{n,\tau}$ (total weight across all estimators/statistics), we have
$w_j(n,\tau;t)\le \rho_{n,\tau}$ for all $(n,\tau)$.
Hence,
\begin{align}\label{eq:corr_term_bound_2}
\sum_{r=1}^m \|C_{n_r,\tau_r}\|_\infty
& \le
\sum_{\tau'\le t}\sum_{n'=1}^N \rho_{n',\tau'}\,\|C_{n',\tau'}\|_\infty\,\mathbf{1}\{k_{n',\tau'}=k\} \nonumber \\
& \le
\Gamma_{\mathrm{eff},k},
\end{align}
where $\Gamma_{\mathrm{eff},k}$ is the (arm-wise) effective corruption defined by
$\Gamma_{\mathrm{eff},k}\triangleq\sum_{\tau=1}^T\sum_{n=1}^N \rho_{n,\tau}\|C_{n,\tau}\|_\infty\mathbf{1}\{k_{n,\tau}=k\}$.
Combining \eqref{eq:corr_term_bound_1}--\eqref{eq:corr_term_bound_2} yields
\begin{equation}\label{eq:corr_term_final}
\left\|\frac{1}{m}\sum_{r=1}^m (\widetilde R_{n_r,\tau_r}-R_{n_r,\tau_r})\right\|_\infty
\;\le\;
\frac{\Gamma_{\mathrm{eff},k}}{m}.
\end{equation}

\subsection{Step 3: Concentration of the Stochastic Term (Coordinate-Wise Hoeffding)}

Fix a coordinate $\ell\in[d]$ and consider the scalar random variables
\[
X_r \;\triangleq\; R_{n_r,\tau_r}^{(\ell)}\in[0,1],
\qquad r=1,\dots,m.
\]
Conditioned on the event $\{k_{n_r,\tau_r}=k\}$ for all $r$ (which holds by construction of $\cI_{j,k}(t)$),
each $X_r$ is distributed according to the $\ell$-th coordinate marginal of $\cD_k$ and has mean $\mu_k^{(\ell)}$.
Moreover, since clean rewards are independent across agents and times conditioned on the chosen arms, the collection $\{X_r\}_{r=1}^m$
is independent conditioned on $\cI_{j,k}(t)$ (even though $\cI_{j,k}(t)$ is random and adaptively chosen).
Therefore, Hoeffding's inequality applies conditionally on $\cI_{j,k}(t)$:
for any $\varepsilon>0$,
\[
\Pr\!\left(
\left|\frac{1}{m}\sum_{r=1}^m (X_r-\mu_k^{(\ell)})\right|\ge \varepsilon
~\Big|~ \cI_{j,k}(t)
\right)
~\le~ 2\exp(-2m\varepsilon^2).
\]
Let $\delta' \triangleq \delta/(dKNT)$ and set
\[
\varepsilon(m) \triangleq \sqrt{\frac{\log(2/\delta')}{2m}}
=
\sqrt{\frac{\log(2dKNT/\delta)}{2m}}.
\]
Then the conditional probability above is at most $\delta'$. Removing the conditioning gives
\begin{equation}\label{eq:hoeff_one}
\Pr\!\left(
\left|\frac{1}{m}\sum_{r=1}^m (R_{n_r,\tau_r}^{(\ell)}-\mu_k^{(\ell)})\right|
\ge
\sqrt{\frac{\log(2dKNT/\delta)}{2m}}
\right)
\;\le\; \delta'.
\end{equation}

\subsection{Step 4: Union Bound Over $(j,k,t,\ell)$}

There are at most $N$ estimators $j$, $K$ arms, $T$ times, and $d$ coordinates.
Applying \eqref{eq:hoeff_one} and a union bound over all $(j,k,t,\ell)$ shows that with probability at least
$1-dKNT\cdot \delta' = 1-\delta$, simultaneously for all $j,k,t$ and all $\ell\in[d]$,
\[
\left|\frac{1}{m}\sum_{r=1}^m (R_{n_r,\tau_r}^{(\ell)}-\mu_k^{(\ell)})\right|
\le
\sqrt{\frac{\log(2dKNT/\delta)}{2m}}.
\]
Equivalently,
\begin{equation}\label{eq:stoch_term_final}
\left\|\frac{1}{m}\sum_{r=1}^m (R_{n_r,\tau_r}-\mu_k)\right\|_\infty
\le
\sqrt{\frac{\log(2dKNT/\delta)}{2m}}.
\end{equation}

\subsection{Step 5: Combine the Bounds}

Plugging \eqref{eq:stoch_term_final} and \eqref{eq:corr_term_final} into \eqref{eq:tri} yields, on the same high-probability event,
\[
\big\|\widehat\mu_{j,k}(t)-\mu_k\big\|_\infty
\le
\sqrt{\frac{\log(2dKNT/\delta)}{2m}} + \frac{\Gamma_{\mathrm{eff},k}}{m}.
\]
Finally, replacing $m$ by $\max\{1,m\}$ accommodates the case $m=0$ discussed in Step~0, giving \eqref{eq:conc_eff}.
\end{proof}

\section{Proof of Lemma \ref{lem:corner_ma}}\label{proofofLemmalem:corner_ma}

\begin{proof}
We first check that $x^{\max}$ is well-defined and belongs to $[0,1]^d$, then show it is (coordinate-wise) an upper bound for $\cC$, and finally use monotonicity of $\phi$ to identify the supremum.

\subsection{Step 1: Well-Definedness of $x^{\max}$ and Basic Properties}

Fix any coordinate $j\in[d]$. Since $\cC\subseteq[0,1]^d$, the set
\[
\cC_j \;\triangleq\; \{x_j:\ x\in\cC\}
\]
is nonempty (unless $\cC=\emptyset$; see the remark at the end) and is bounded above by $1$.
Hence the supremum
\[
x^{\max}_j \;\triangleq\; \sup \cC_j
\]
satisfies $0\le x^{\max}_j\le 1$. Therefore $x^{\max}\in[0,1]^d$.

\subsection{Step 2: $x^{\max}$ Dominates Every Point in $\cC$}

By definition of supremum, for every $x\in\cC$ and every coordinate $j$,
\[
x_j \le \sup\{z_j:\ z\in\cC\} = x^{\max}_j.
\]
Thus, $x\preceq x^{\max}$ holds for all $x\in\cC$.

\subsection{Step 3: Upper Bound on the Objective: $\sup_{x\in\cC}\phi(x)\le \phi(x^{\max})$}

Since $\phi$ is coordinate-wise nondecreasing, $x\preceq x^{\max}$ implies $\phi(x)\le \phi(x^{\max})$.
Applying this to every $x\in\cC$ yields
\[
\sup_{x\in\cC}\phi(x) \;\le\; \phi(x^{\max}).
\]

\subsection{Step 4: Achievability of $\phi(x^{\max})$ as the Supremum}

We show that values of $\phi$ on $\cC$ can get arbitrarily close to $\phi(x^{\max})$.

Fix $\varepsilon>0$. For each coordinate $j\in[d]$, by the definition of $x^{\max}_j$ as a supremum, there exists $x^{(j)}\in\cC$ such that
\[
x^{(j)}_j \;\ge\; x^{\max}_j - \varepsilon.
\]
Define the coordinate-wise maximum (join) $y \triangleq \bigvee_{j=1}^d x^{(j)}$ where $y_\ell \triangleq \max_{j\in[d]} x^{(j)}_\ell \ \ \forall \ell\in[d]$. Then, for each coordinate $\ell$,
\[
y_\ell \;\ge\; x^{(\ell)}_\ell \;\ge\; x^{\max}_\ell-\varepsilon,
\]
so
\[
x^{\max}-\varepsilon\mathbf{1} \;\preceq\; y \;\preceq\; \mathbf{1}.
\]

It remains to justify $y\in\cC$. This is where \emph{upper-closedness} is used.
Pick any $x^{(1)}\in\cC$ (one of the selected points). By construction, $x^{(1)}\preceq y\preceq \mathbf{1}$.
Since $\cC$ is upper-closed, $x^{(1)}\in\cC$ and $x^{(1)}\preceq y\preceq\mathbf{1}$ imply $y\in\cC$.

Now apply monotonicity of $\phi$:
\[
\phi(y)\;\ge\;\phi(x^{\max}-\varepsilon\mathbf{1}).
\]
Taking supremum over $x\in\cC$ gives
\[
\sup_{x\in\cC}\phi(x) \;\ge\; \phi(y) \;\ge\; \phi(x^{\max}-\varepsilon\mathbf{1}).
\]

Finally, since $\phi$ is $L$-Lipschitz under $\|\cdot\|_\infty$ on $[0,1]^d$,
\[
\big|\phi(x^{\max})-\phi(x^{\max}-\varepsilon\mathbf{1})\big|
\;\le\;
L\| \varepsilon\mathbf{1}\|_\infty
\;=\;
L\varepsilon,
\]
so
\[
\phi(x^{\max}-\varepsilon\mathbf{1})
\;\ge\;
\phi(x^{\max})-L\varepsilon.
\]
Therefore,
\[
\sup_{x\in\cC}\phi(x)
\;\ge\;
\phi(x^{\max})-L\varepsilon.
\]
Since $\varepsilon>0$ is arbitrary, letting $\varepsilon\downarrow 0$ yields
\[
\sup_{x\in\cC}\phi(x)\;\ge\;\phi(x^{\max}).
\]

\subsection{Step 5: Combine}

Together with Step 3, we conclude $\sup_{x\in\cC}\phi(x)=\phi(x^{\max})$.

\paragraph{Specialization to $\ell_\infty$ rectangles}
Let $\cC=\{x:\|x-\widehat\mu\|_\infty\le\beta\}\cap[0,1]^d$.
Then for each coordinate $j$,
\[
\sup\{x_j:\ x\in\cC\} = \min\{1,\widehat\mu_j+\beta\},
\]
because the $\ell_\infty$ constraint allows increasing coordinate $j$ up to $\widehat\mu_j+\beta$, and the intersection with $[0,1]^d$
caps it at $1$. Hence
\[
x^{\max} = \Pi_{[0,1]^d}(\widehat\mu+\beta\mathbf{1}),
\]
and the first part implies
\[
\sup_{x\in\cC}\phi(x)=\phi\!\big(\Pi_{[0,1]^d}(\widehat\mu+\beta\mathbf{1})\big).
\]

\paragraph{Remark (empty set)}
If $\cC=\emptyset$, then $\sup_{x\in\cC}\phi(x)=-\infty$ by convention, while $x^{\max}$ is undefined.
In our application, confidence sets are nonempty (they at least intersect $[0,1]^d$), so this case does not arise.
\end{proof}

\section{Proof of Theorem~\ref{thm:meta_team_ub}}\label{proofofLemmathm:meta_team_ub}

\begin{proof}
Fix $\delta\in(0,1)$. Let $\cE$ denote the ''good`` event on which the vector mean concentration bound in
Lemma~\ref{lem:conc_eff} holds simultaneously for all estimators $j$, all arms $k$, and all times $t$; namely,
\begin{equation}\label{eq:good_event}
\big\|\widehat\mu_{j,k}(t)-\mu_k\big\|_\infty
\le
\beta_{j,k}(t)
\qquad
\forall (j,k,t),
\end{equation}
where $\beta_{j,k}(t)$ is the radius used by the protocol (in particular, it upper bounds the stochastic fluctuation term
and the bias term induced by the arm-wise effective corruption $\Gamma_{\mathrm{eff},k}$).
By Lemma~\ref{lem:conc_eff}, $\Pr(\cE)\ge 1-\delta$.

Throughout the proof, we condition on $\cE$.

\subsection{Step 1: UCB Indices are Valid Upper Bounds on $\theta_k=\phi(\mu_k)$}

A UCB-type protocol maintains, for each estimator $j$ and arm $k$ at time $t$, an empirical mean $\widehat\mu_{j,k}(t)$
and uses the optimistic index
\begin{equation}\label{eq:ucb_index_proof}
U_{j,k}(t)
~\triangleq~
\phi\!\Big(\Pi_{[0,1]^d}\big(\widehat\mu_{j,k}(t)+\beta_{j,k}(t)\mathbf{1}\big)\Big).
\end{equation}
On $\cE$, \eqref{eq:good_event} implies coordinate-wise
\[
\mu_k \preceq \widehat\mu_{j,k}(t)+\beta_{j,k}(t)\mathbf{1}.
\]
Since $\mu_k\in[0,1]^d$, applying coordinate-wise projection preserves dominance:
\[
\mu_k \preceq \Pi_{[0,1]^d}\big(\widehat\mu_{j,k}(t)+\beta_{j,k}(t)\mathbf{1}\big).
\]
Because $\phi$ is coordinate-wise nondecreasing, for all $j,k,t$,
\begin{equation}\label{eq:ucb_valid}
\theta_k=\phi(\mu_k)
\le
\phi\!\Big(\Pi_{[0,1]^d}\big(\widehat\mu_{j,k}(t)+\beta_{j,k}(t)\mathbf{1}\big)\Big)
=
U_{j,k}(t).
\end{equation}

\subsection{Step 2: Instantaneous Team Regret is Controlled by the Chosen Radius}

Fix an agent-round decision $(n,t)$, and let $j=j(n,t)$ be the estimator used to form the arm index at that decision.
Let $k_{n,t}$ be the arm selected by the protocol. Since the protocol is UCB-type,
\[
k_{n,t}\in\arg\max_{k\in[K]} U_{j,k}(t).
\]
Let $k^*\in\arg\max_k \theta_k$ be an optimal arm. Then
\[
U_{j,k_{n,t}}(t)\ge U_{j,k^*}(t)\ge \theta_{k^*},
\]
where the last inequality uses \eqref{eq:ucb_valid}. Therefore the instantaneous scalarized regret satisfies
\begin{align}
\theta_{k^*}-\theta_{k_{n,t}}
&\le
U_{j,k_{n,t}}(t)-\theta_{k_{n,t}}.
\label{eq:inst_regret_reduce}
\end{align}
We now bound $U_{j,k}(t)-\theta_k$ in terms of $\beta_{j,k}(t)$.
Using \eqref{eq:ucb_index_proof}, Lipschitzness of $\phi$ under $\|\cdot\|_\infty$, and nonexpansiveness of projection
under $\|\cdot\|_\infty$,
\begin{align}
U_{j,k}(t)-\theta_k
&=
\phi\!\Big(\Pi_{[0,1]^d}\big(\widehat\mu_{j,k}(t)+\beta_{j,k}(t)\mathbf{1}\big)\Big) - \phi(\mu_k)
\nonumber\\
&\le
L\Big\|
\Pi_{[0,1]^d}\big(\widehat\mu_{j,k}(t)+\beta_{j,k}(t)\mathbf{1}\big)-\mu_k
\Big\|_\infty
\nonumber\\
&\le
L\Big\|
\widehat\mu_{j,k}(t)+\beta_{j,k}(t)\mathbf{1}-\mu_k
\Big\|_\infty
\nonumber\\
&\le
L\big\|\widehat\mu_{j,k}(t)-\mu_k\big\|_\infty + L\,\beta_{j,k}(t)
\;\le\;
2L\,\beta_{j,k}(t),
\label{eq:ucb_gap_bound}
\end{align}
where the last step uses \eqref{eq:good_event}.
Combining \eqref{eq:inst_regret_reduce} and \eqref{eq:ucb_gap_bound} yields, for every agent-round $(n,t)$,
\begin{equation}\label{eq:inst_regret_beta}
\theta_{k^*}-\theta_{k_{n,t}}
\le
2L\,\beta_{j(n,t),\,k_{n,t}}(t).
\end{equation}

\subsection{Step 3: Summing Over all $NT$ Pulls and Bounding the Stochastic Part}

Summing \eqref{eq:inst_regret_beta} over all $(n,t)$ gives
\begin{equation}\label{eq:sum_beta_start}
\Reg_\phi^{\mathrm{team}}(T)
\le
2L\sum_{t=1}^T\sum_{n=1}^N \beta_{j(n,t),\,k_{n,t}}(t).
\end{equation}
We upper bound the RHS using the explicit form of the radii (as in Lemma~\ref{lem:conc_eff}):
\[
\beta_{j,k}(t)
\;\le\;
\sqrt{\frac{\log(2dKNT/\delta)}{2\,\max\{1,|\cI_{j,k}(t)|\}}}
+\frac{\Gamma_{\mathrm{eff},k}}{\max\{1,|\cI_{j,k}(t)|\}}.
\]
For the stochastic term, define $m\mapsto a(m)\triangleq \sqrt{\log(2dKNT/\delta)/(2m)}$.
Every time an estimator uses an arm-$k$ observation as a unit-weight sample, the corresponding counter $|\cI_{j,k}(\cdot)|$ increases by $1$.
Hence, when we sum $a(|\cI_{j,k}(t)|)$ over all agent-rounds in which arm $k$ is played (and the chosen estimator uses its own observation),
the denominators can be pessimistically upper bounded by the harmonic sequence $1,2,\dots,N_k^{\mathrm{team}}(T)$.
Therefore,
\begin{align}
& \sum_{t=1}^T\sum_{n=1}^N
\sqrt{\frac{\log(2dKNT/\delta)}{2\,\max\{1,|\cI_{j(n,t),k_{n,t}}(t)|\}}} \nonumber \\
& \le
\sqrt{\frac{\log(2dKNT/\delta)}{2}}
\sum_{k=1}^K\sum_{r=1}^{N_k^{\mathrm{team}}(T)} \frac{1}{\sqrt{r}}
\nonumber\\
&\le
\sqrt{\frac{\log(2dKNT/\delta)}{2}}
\sum_{k=1}^K 2\sqrt{N_k^{\mathrm{team}}(T)}
\nonumber\\
&\le
2\sqrt{\frac{\log(2dKNT/\delta)}{2}}
\sqrt{K\sum_{k=1}^K N_k^{\mathrm{team}}(T)}
\nonumber\\
&=
2\sqrt{\frac{\log(2dKNT/\delta)}{2}}\sqrt{KNT},
\label{eq:stoch_sum_bound}
\end{align}
where we used $\sum_{r=1}^m r^{-1/2}\le 2\sqrt{m}$ and Cauchy--Schwarz.

\subsection{Step 4: Bounding the Corruption Part via a Harmonic Sum}

Similarly, for the corruption term,
\begin{align}
& \sum_{t=1}^T\sum_{n=1}^N
\frac{\Gamma_{\mathrm{eff},k_{n,t}}}{\max\{1,|\cI_{j(n,t),k_{n,t}}(t)|\}}
\le
\sum_{k=1}^K \Gamma_{\mathrm{eff},k}
\sum_{r=1}^{N_k^{\mathrm{team}}(T)} \frac{1}{r}
\nonumber\\
&\le
\sum_{k=1}^K \Gamma_{\mathrm{eff},k}\,\log\!\big(1+N_k^{\mathrm{team}}(T)\big)
\nonumber\\
&\le
\log(1+NT)\sum_{k=1}^K \Gamma_{\mathrm{eff},k}
=
\Gamma_{\mathrm{eff}}\log(1+NT),
\label{eq:corr_sum_bound}
\end{align}
where we used $\sum_{r=1}^m \frac1r \le \log(1+m)$ and $\sum_k \Gamma_{\mathrm{eff},k}=\Gamma_{\mathrm{eff}}$.

\subsection{Step 5: Combine}

Plugging \eqref{eq:stoch_sum_bound} and \eqref{eq:corr_sum_bound} into \eqref{eq:sum_beta_start} yields
\begin{align}
& \Reg_\phi^{\mathrm{team}}(T) \nonumber \\
& \le
2L\Bigg(
2\sqrt{\frac{\log(2dKNT/\delta)}{2}}\sqrt{KNT}
+
\Gamma_{\mathrm{eff}}\log(1+NT)
\Bigg).
\end{align}
Absorbing numerical factors into a universal constant $c>0$ gives \eqref{eq:teamub_meta}.
\end{proof}

\section{Proof of Theorem~\ref{thm:amp_lb}}\label{proofofTheoremthm:amp_lb}

\begin{proof}
We prove the claim for the scalar case $d=1$ with $\phi(x)=x$ (so $\Reg_\phi^{\mathrm{team}}(T)$ is the standard team regret).
This immediately implies the stated lower bound for general coordinate-wise nondecreasing $L$-Lipschitz scalarizations by scaling constants
(e.g., by taking instances where only one coordinate varies and using $L$-Lipschitzness).

Let $H \triangleq NT$ denote the total number of \emph{agent-round pulls}. Throughout, we work under the (S1) model:
at the end of each round $t$, every agent broadcasts its raw observation, and under \emph{append-all} each agent appends every received raw sample
as a fresh sample into its local dataset. Thus, for any agent-round observation $(n,t)$, the realized (possibly corrupted) feedback is present in the
local history of \emph{all} $N$ agents from the end of round $t$ onward.

The proof has two components:
(i) the classical stochastic minimax term $\Omega(\sqrt{KH})$ (already present when $\Gamma=0$), and
(ii) an additive corruption term $\Omega(N\Gamma)$ induced by (S1)-replication.
Combining them via $\max\{a,b\}\ge \frac12(a+b)$ yields the stated form with universal constants.

\subsection{Step 1: The Stochastic Term $\Omega(\sqrt{KH})$}

Consider the clean case $\Gamma=0$.
It is standard (see, e.g., \cite{auer2002,bubeck2012}) that for any (possibly adaptive) team policy making $H$ total pulls
there exists a $K$-armed Bernoulli instance such that
\begin{equation}\label{eq:clean_lb}
\mathbb{E}\big[\Reg^{\mathrm{team}}(T)\big]\ \ge\ c_{\mathrm{sto}}\sqrt{KH}
\end{equation}
for a universal constant $c_{\mathrm{sto}}>0$.
(One convenient construction is: pick a uniformly random optimal arm $k^\star\sim\mathrm{Unif}([K])$,
set $\mu_{k^\star}=\tfrac12+\Delta$ and $\mu_k=\tfrac12$ for $k\neq k^\star$, then choose $\Delta\asymp\sqrt{K/H}$
and apply a change-of-measure/Fano argument.)

\subsection{Step 2: A Corruption Term $\Omega(N\Gamma)$ Under (S1) Append-All}

We now construct a family of two-arm instances and an adversary which, under (S1) append-all, makes the team essentially unable to
identify the better arm unless the team pays regret linear in $N\Gamma$.

\smallskip
\noindent\textbf{Two-point instance.}
Let $K=2$ and consider two bandit instances $\mathcal{I}_+$ and $\mathcal{I}_-$:
\[
\mathcal{I}_+:\quad \mu_1=\tfrac12+\Delta,\ \mu_2=\tfrac12,
\qquad
\mathcal{I}_-:\quad \mu_1=\tfrac12,\ \mu_2=\tfrac12+\Delta,
\]
with $\Delta\in(0,1/8]$ to be chosen later. Rewards are Bernoulli with these means.

\smallskip
\noindent\textbf{A coupling-style corruption strategy.}
We describe an adversary that, in either world $\mathcal{I}_+$ or $\mathcal{I}_-$, corrupts \emph{only the samples from the currently-better arm}
so that the \emph{observed} reward distribution of that arm is pushed down to $\mathrm{Bernoulli}(\tfrac12)$.
Concretely, suppose the clean reward is $R\in\{0,1\}$ from a better arm with mean $\tfrac12+\Delta$.
Let the adversary produce the corrupted observation $\widetilde R$ by
\[
\widetilde R \;=\; R\cdot Z,\qquad Z\sim\mathrm{Bernoulli}\!\left(\frac{\tfrac12}{\tfrac12+\Delta}\right)
\ \text{independent of }R.
\]
Then $\mathbb{P}(\widetilde R=1)=(\tfrac12+\Delta)\cdot \frac{\tfrac12}{\tfrac12+\Delta}=\tfrac12$, i.e.,
$\widetilde R\sim\mathrm{Bernoulli}(\tfrac12)$.
This transformation flips a $1$ to $0$ with probability
\begin{align}
& \mathbb{P}(\widetilde R\neq R)=\mathbb{P}(R=1)\cdot \mathbb{P}(Z=0\mid R=1) \nonumber \\
& =(\tfrac12+\Delta)\left(1-\frac{\tfrac12}{\tfrac12+\Delta}\right)=\Delta.
\end{align}
Hence the expected \emph{per-sample} corruption magnitude is
$\mathbb{E}[\,|\widetilde R-R|\,]=\Delta$.
(For samples from the worse arm with mean $\tfrac12$, the adversary does nothing.)

\smallskip
\noindent\textbf{Budget accounting under (S1) append-all.}
Let $M$ be the total number of pulls of the better arm across all agents and all rounds.
Under the above strategy, the \emph{physical} corruption budget used satisfies
\[
\mathbb{E}\!\left[\sum_{t=1}^T\sum_{n=1}^N |C_{n,t}|\right]
=\mathbb{E}\!\left[\sum_{\text{better-arm pulls}} |\widetilde R-R|\right]
=\Delta\cdot \mathbb{E}[M].
\]
Now use the (S1) append-all property: each corrupted sample $(n,t)$ is broadcast and appended by \emph{all} $N$ agents,
so it appears $N$ times in the \emph{joint} team transcript (the collection of all agents' local histories that drive their future actions).
Equivalently, one unit of physical corruption on a single agent-round observation induces $N$ units of ``statistical damage'' across the team’s
decision rules. This is exactly the replication phenomenon quantified by $\rho_{n,t}=N$ in Lemma~\ref{lem:protocol_amp_clean}.

Formally, define the \emph{replicated} (or effective) corruption cost of this adversary by
\[
B \;\triangleq\; N\sum_{t=1}^T\sum_{n=1}^N |C_{n,t}|.
\]
Then $\mathbb{E}[B]=N\Delta\,\mathbb{E}[M]$.

Choose
\begin{equation}\label{eq:Delta_choice}
\Delta \;\triangleq\; \min\left\{\frac18,\ \frac{N\Gamma}{8H}\right\}.
\end{equation}
Since $M\le H$, we have $\mathbb{E}[B]\le N\Delta H \le N\cdot \frac{N\Gamma}{8H}\cdot H = \frac{N^2\Gamma}{8}$,
and in particular $\mathbb{E}\big[\sum_{t,n}|C_{n,t}|\big]\le \Gamma/8$ whenever $\Delta=\frac{N\Gamma}{8H}$.
By Markov's inequality,
\[
\mathbb{P}\!\left(\sum_{t,n}|C_{n,t}|\le \Gamma\right)\ \ge\ 1-\frac{\mathbb{E}[\sum_{t,n}|C_{n,t}|]}{\Gamma}\ \ge\ \frac78.
\]
So with constant probability the strategy is feasible under the budget $\Gamma$.

\smallskip
\noindent\textbf{Indistinguishability.}
Under $\mathcal{I}_+$, the adversary pushes the observed rewards of arm~$1$ down to $\mathrm{Bernoulli}(\tfrac12)$
and leaves arm~$2$ at $\mathrm{Bernoulli}(\tfrac12)$.
Under $\mathcal{I}_-$, it pushes the observed rewards of arm~$2$ down to $\mathrm{Bernoulli}(\tfrac12)$
and leaves arm~$1$ at $\mathrm{Bernoulli}(\tfrac12)$.
Therefore, \emph{conditioned on the event that the budget constraint is satisfied}, the full observed transcript
(actions, broadcasts, and received observations) has the same distribution under $\mathcal{I}_+$ and $\mathcal{I}_-$.
Hence no (possibly randomized) cooperative policy can identify the better arm with probability exceeding $1/2$ in both instances.

\smallskip
\noindent\textbf{Regret consequence.}
Let $\pi$ be any cooperative protocol under (S1) append-all. By the above indistinguishability,
for the uniform prior over $\{\mathcal{I}_+,\mathcal{I}_-\}$ we have
\begin{align}
& \mathbb{E}\big[\Reg^{\mathrm{team}}(T)\big] \nonumber \\
& \ge \frac12\cdot \Delta\cdot \mathbb{E}\big[\#\{\text{suboptimal pulls across all }H\text{ pulls}\}\big] \nonumber \\
& \ge c'\Delta H
\end{align}
for a universal constant $c'>0$ (e.g., $c'=1/8$ suffices after conditioning on budget-feasibility with probability $\ge 7/8$).
Using $\Delta=\frac{N\Gamma}{8H}$ (the non-saturated case in \eqref{eq:Delta_choice}), we get
\begin{equation}\label{eq:corrupt_lb}
\mathbb{E}\big[\Reg^{\mathrm{team}}(T)\big]\ \ge\ c_{\mathrm{cor}}\,N\Gamma
\end{equation}
for a universal constant $c_{\mathrm{cor}}>0$, as long as $N\Gamma\le H$ (otherwise the trivial bound
$\Reg^{\mathrm{team}}(T)\le H$ dominates and one can replace $N\Gamma$ by $\min\{N\Gamma,H\}$).

\subsection{Step 3: Combine}

From \eqref{eq:clean_lb} and \eqref{eq:corrupt_lb}, the worst-case expected team regret under (S1) append-all satisfies
\begin{align}
& \sup_{\text{instances}}\ \inf_{\text{adv: }\sum_{t,n}|C_{n,t}|\le\Gamma}\ \mathbb{E}[\Reg^{\mathrm{team}}(T)] \nonumber \\
& \ge\ \max\Big\{c_{\mathrm{sto}}\sqrt{KH},\ c_{\mathrm{cor}}N\Gamma\Big\} \nonumber \\
& \ge\ \tfrac12 c_{\mathrm{sto}}\sqrt{KH}+\tfrac12 c_{\mathrm{cor}}N\Gamma.
\end{align}
Absorbing the factor $1/2$ into constants yields the claim with $H=NT$.
\end{proof}

\section{Proof of Lemma~\ref{lem:cert}}\label{proofofLemmalem:cert}

\begin{proof}
Fix an arm $k\in[K]$. Let $\cT_k^{\mathrm{ver}}(t)$ be the set of verified agent-round indices up to time $t$ in which arm $k$ was pulled $\cT_k^{\mathrm{ver}}(t)
~\triangleq~
\{(n,\tau): 1\le \tau\le t,\ k_{n,\tau}=k,\ V_{n,\tau}=1\}$, $H_k^{\mathrm{ver}}(t)=|\cT_k^{\mathrm{ver}}(t)|$. Define the verified-only empirical mean
\[
\widehat\mu_k^{\mathrm{ver}}(t)
~\triangleq~
\frac{1}{\max\{1,H_k^{\mathrm{ver}}(t)\}}
\sum_{(n,\tau)\in\cT_k^{\mathrm{ver}}(t)} R_{n,\tau}.
\]

\subsection{Step 1: Verified Samples are i.i.d. From $\cD_k$}

For each agent-round $(n,t)$, the verification decision $V_{n,t}$ is $\sigma(\cH_{n,t-1})$-measurable and is made \emph{before observing}
the reward feedback at time $t$. Hence $V_{n,t}$ is independent of the realized reward $R_{n,t}$ conditional on the chosen arm $k_{n,t}$.
Therefore, conditional on the event $\{k_{n,t}=k,\ V_{n,t}=1\}$, the random vector $R_{n,t}$ is distributed as $\cD_k$.
Moreover, across distinct agent-rounds, the reward draws are independent (conditioned on actions).
Consequently, for any fixed $t$, the multiset $\{R_{n,\tau} : (n,\tau)\in\cT_k^{\mathrm{ver}}(t)\}$ consists of
$H_k^{\mathrm{ver}}(t)$ independent samples from $\cD_k$.

\subsection{Step 2: Coordinate-Wise Hoeffding and Union Bound Over $(k,j,s)$}

Fix a coordinate $j\in[d]$. For any integer $s\ge 1$, let
\[
\bar R_{k,j}(s)\triangleq \frac{1}{s}\sum_{\ell=1}^s X_{k,j,\ell},
\]
where $X_{k,j,\ell}\in[0,1]$ are i.i.d.\ with mean $\mu_k^{(j)}$ (a generic copy of the verified samples of arm $k$ in coordinate $j$).
Hoeffding's inequality gives, for any $\varepsilon>0$,
\[
\Pr\!\left(\big|\bar R_{k,j}(s)-\mu_k^{(j)}\big|\ge \varepsilon\right)\le 2e^{-2s\varepsilon^2}.
\]
Set
\[
\varepsilon(s)\triangleq \sqrt{\frac{\log(2dKNT/\delta)}{2s}}.
\]
Then
\begin{align}
& \Pr\!\left(\big|\bar R_{k,j}(s)-\mu_k^{(j)}\big|\ge \varepsilon(s)\right) \nonumber \\
& \le 2\exp\!\Big(-2s\cdot \frac{\log(2dKNT/\delta)}{2s}\Big) = \frac{\delta}{dKNT}.
\end{align}
Applying a union bound over all arms $k\in[K]$, coordinates $j\in[d]$, and sample sizes $s\in\{1,2,\dots,NT\}$ yields an event $\cE$
with $\Pr(\cE)\ge 1-\delta$ such that on $\cE$, for all $k\in[K],~j\in[d],~s\in[NT]$,
\begin{equation}\label{eq:coord_unif}
\big|\bar R_{k,j}(s)-\mu_k^{(j)}\big|
\le \sqrt{\frac{\log(2dKNT/\delta)}{2s}}.
\end{equation}

\subsection{Step 3: Translate \eqref{eq:coord_unif} to the Random-Time Estimator $\widehat\mu_k^{\mathrm{ver}}(t)$}

Fix any $k$ and $t$.
If $H_k^{\mathrm{ver}}(t)=s\ge 1$, then by Step 1, $\widehat\mu_k^{\mathrm{ver}}(t)$ is exactly the empirical mean of $s$ i.i.d.\ samples
from $\cD_k$; thus \eqref{eq:coord_unif} implies simultaneously for all $j\in[d]$,
\[
\big|\widehat\mu_{k}^{\mathrm{ver}}(t)^{(j)}-\mu_k^{(j)}\big|
\le \sqrt{\frac{\log(2dKNT/\delta)}{2s}}.
\]
Equivalently,
\[
\big\|\widehat\mu_{k}^{\mathrm{ver}}(t)-\mu_k\big\|_\infty
\le \sqrt{\frac{\log(2dKNT/\delta)}{2\,H_k^{\mathrm{ver}}(t)}}.
\]
(When $H_k^{\mathrm{ver}}(t)=0$, the verified-only estimate is undefined/unused for certification; in implementations one can set the
certificate radius to $+\infty$ and treat the interval as vacuous. If you keep the $\max\{1,H_k^{\mathrm{ver}}(t)\}$ convention,
the displayed bound remains valid for all $t$ after defining $\widehat\mu_k^{\mathrm{ver}}(t)$ arbitrarily when $H_k^{\mathrm{ver}}(t)=0$
and taking the confidence interval as vacuous.)

\subsection{Step 4: Apply Lipschitzness of $\phi$}

Since $\phi$ is $L$-Lipschitz under $\|\cdot\|_\infty$,
\begin{align}
& \big|\phi(\widehat\mu_k^{\mathrm{ver}}(t))-\phi(\mu_k)\big| \nonumber \\
& \le L\big\|\widehat\mu_k^{\mathrm{ver}}(t)-\mu_k\big\|_\infty \nonumber \\
& \le L\sqrt{\frac{\log(2dKNT/\delta)}{2\,\max\{1,H_k^{\mathrm{ver}}(t)\}}}.
\end{align}
Recalling $\theta_k=\phi(\mu_k)$ completes the proof.
\end{proof}

\section{Proof of Theorem~\ref{thm:verified_sharing}}\label{proofofTheoremthm:verified_sharing}

\begin{proof}
We describe an explicit cooperative protocol (using (S3) recommendation-only messages, with optional verified certificates) and then prove the two claims.

\subsection{Algorithm: S3 with Certified Override (S3-Cert)}

Fix confidence level $\delta\in(0,1)$. The protocol maintains two tracks:

\smallskip
\noindent
\emph{(A) Always-valid baseline track.}
All agents run a fixed recommendation-only coordination protocol $\mathsf{Comm}$ (e.g., an epoch-based coordination rule)
whose arm-selection rule is UCB-type and uses only local (possibly corrupted) observations.
Each agent broadcasts only an arm recommendation each round (as in (S3)).
Let $\Reg^{\mathrm{clean}}_{\mathrm{comm}}(K,N,T)$ denote the (high-probability) regret overhead of the same protocol in the clean case.
(Concretely, in our paper this is the ``clean-case coordination'' term of $\mathsf{Comm}$.)

\smallskip
\noindent
\emph{(B) Certified-override track.}
The team spends up to $\nu$ verifications to create \emph{verified-only} statistics and a \emph{certificate-based stopping rule}.
Let $H_k^{\mathrm{ver}}(t)$ and $\widehat\mu_k^{\mathrm{ver}}(t)$ be the verified-only count and empirical mean
(defined in the main text) aggregated over all agents up to time $t$.
At each verified agent-round $(n,t)$ with $V_{n,t}=1$, agent $n$ broadcasts the tuple
\[
(k_{n,t},R_{n,t},V_{n,t}=1),
\]
so that all agents can update the common verified-only statistics\footnote{This broadcast occurs only on verified rounds, at most $\nu$ times, and is therefore the ``optional verified certificate'' channel.}.
For each arm $k$ and time $t$, define the certified radius $\eps_k(t)\triangleq L\sqrt{\frac{\log(2dKNT/\delta)}{2\max\{1,H_k^{\mathrm{ver}}(t)\}}}$, $\LCB_k(t)\triangleq \phi(\widehat\mu_k^{\mathrm{ver}}(t))-\eps_k(t),
\UCB_k(t)\triangleq \phi(\widehat\mu_k^{\mathrm{ver}}(t))+\eps_k(t)$. The protocol declares a \emph{certified winner} at the first time
\[
t_{\mathrm{cert}}\triangleq \inf\Big\{t:\ \exists\,k\in[K]\ \text{s.t.}\ \LCB_k(t)\ \ge\ \max_{k'\neq k}\UCB_{k'}(t)\Big\}.
\]
If such $t_{\mathrm{cert}}\le T$ occurs, all agents commit from time $t_{\mathrm{cert}}+1$ onward to the certified winner
\[
\widehat k \in \arg\max_{k\in[K]}\LCB_k(t_{\mathrm{cert}}).
\]
If no certified winner appears by time $T$, agents simply keep following the baseline track.

\smallskip
By construction, (S3-Cert) uses recommendation-only sharing on unverified rounds, and uses at most $\nu$ verified broadcasts.
We now prove the two parts.

\subsection{Part (1): Always-Valid Robust Baseline}

Consider any instance and any adaptive adversary satisfying the global budget~\eqref{eq:global_budget}.
On rounds before commitment, the arm-selection decisions are exactly those of the baseline protocol $\mathsf{Comm}$
(possibly with \emph{more accurate} observations on verified rounds, since verified feedback is clean).
Thus, any regret upper bound proven for $\mathsf{Comm}$ under corruption continues to hold for (S3-Cert).

In particular, instantiate $\mathsf{Comm}$ as the recommendation-only robust-UCB-type protocol analyzed in our paper.
Under (S3), no raw rewards are replicated across agents, so by Lemma~\ref{lem:protocol_amp_clean} we have
$\Gamma_{\mathrm{eff}}=\Gamma$. Applying Theorem~\ref{thm:meta_team_ub} (team regret bound via effective corruption)
to the baseline track yields that on an event of probability at least $1-\delta$,
\[
\Reg_\phi^{\mathrm{team}}(T)
\ \le\ \widetilde O\!\Big(L\cdot \Reg^{\mathrm{clean}}_{\mathrm{comm}}(K,N,T) + L\Gamma\Big),
\]
where $\Reg^{\mathrm{clean}}_{\mathrm{comm}}(K,N,T)$ denotes the clean-case coordination overhead of the same protocol.
Since the certified-override can only (i) switch to a single arm and (ii) uses only clean information to decide the switch,
it cannot increase regret on the high-probability event on which the baseline bound holds.
This proves Part~(1).

\subsection{Part (2): Verification-Driven Override and $\Gamma$-Independence}

Assume $\Delta_{\min}\triangleq \min_{k\neq k^*}\Delta_k>0$ and the verification budget satisfies~\eqref{eq:cert_thresh}.
Let $\cE_{\mathrm{ver}}$ be the high-probability event from Lemma~\ref{lem:cert} applied to the verified-only statistics:
with probability at least $1-\delta$, simultaneously for all arms $k$ and all times $t$,
\begin{equation}\label{eq:cert_event}
\Big|\phi(\widehat\mu_k^{\mathrm{ver}}(t))-\theta_k\Big|\ \le\ \eps_k(t).
\end{equation}
We condition on $\cE_{\mathrm{ver}}$ for the remainder of the argument.

\smallskip
\noindent
\emph{Step 1: Once verified counts are large enough, the certified winner is $k^*$.}
Let $s\triangleq \min_{k\in[K]} H_k^{\mathrm{ver}}(t)$ be the minimum verified count across arms at some time $t$.
If $s$ is large enough so that
\begin{equation}\label{eq:eps_small}
\max_{k\in[K]}\eps_k(t)\ \le\ \frac{\Delta_{\min}}{4},
\end{equation}
then the certification rule must accept $k^*$ at (or before) time $t$.
Indeed, on $\cE_{\mathrm{ver}}$ we have for the optimal arm $k^*$,
\begin{align}
& \LCB_{k^*}(t)\ =\ \phi(\widehat\mu_{k^*}^{\mathrm{ver}}(t))-\eps_{k^*}(t)
\nonumber \\
& \ge\ \theta_{k^*}-2\eps_{k^*}(t)
\ \ge\ \theta_{k^*}-2\max_{k}\eps_k(t),
\end{align}
and for any suboptimal arm $k\neq k^*$,
\begin{align}
& \UCB_k(t)\ =\ \phi(\widehat\mu_k^{\mathrm{ver}}(t))+\eps_k(t)
\nonumber \\
& \le\ \theta_k+2\eps_k(t)
\ \le\ \theta_{k^*}-\Delta_{\min}+2\max_k\eps_k(t).
\end{align}
If \eqref{eq:eps_small} holds, then for every $k\neq k^*$,
\begin{align}
& \LCB_{k^*}(t) \ge \theta_{k^*}-\frac{\Delta_{\min}}{2}
\nonumber \\
& > \theta_{k^*}-\Delta_{\min}+\frac{\Delta_{\min}}{2} \ge \UCB_k(t),
\end{align}
so $\LCB_{k^*}(t)\ge \max_{k\neq k^*}\UCB_k(t)$ and the rule certifies $k^*$.

\smallskip
\noindent
\emph{Step 2: The verification budget threshold implies \eqref{eq:eps_small}.}
Under the condition \eqref{eq:cert_thresh}, it is possible to allocate verifications so that each arm receives
\[
H_k^{\mathrm{ver}}(t)\ \ge\ \left\lfloor \frac{\nu}{K}\right\rfloor
\]
e.g., round-robin verification across arms and agents, by some time $t\le T$ (assuming $\nu\le NT$; otherwise verification can cover all pulls and the claim is trivial).
For such a time $t$, we have for all $k$,
\[
\eps_k(t)\ \le\ L\sqrt{\frac{\log(2dKNT/\delta)}{2\lfloor \nu/K\rfloor}}
\ \le\ \frac{\Delta_{\min}}{4}
\]
by choosing the universal constant $c$ in \eqref{eq:cert_thresh} large enough (e.g., $c\ge 8$ suffices up to floor/ceiling effects).
Thus \eqref{eq:eps_small} holds, so by Step~1 the certified winner is $k^*$ and commitment occurs at some $t^*\le T$.

\smallskip
\noindent
\emph{Step 3: Regret decomposition and independence from $\Gamma$.}
Let $t^*$ be the (random) commitment time. After time $t^*$, all agents play $k^*$, hence incur zero regret.
Therefore,
\begin{align}
& \Reg_\phi^{\mathrm{team}}(T)\ =\ \sum_{t=1}^{t^*}\sum_{n=1}^N \big(\theta_{k^*}-\theta_{k_{n,t}}\big)
\nonumber \\
& \le\ \Delta_{\max}\cdot \sum_{t=1}^{t^*}\sum_{n=1}^N \mathbf{1}\{k_{n,t}\neq k^*\}.
\end{align}
Split the pre-commitment pulls into (i) verified pulls and (ii) all other pulls:
\begin{align}
& \Reg_\phi^{\mathrm{team}}(T)
\ \le\ \Delta_{\max}\cdot \underbrace{\sum_{t\le t^*}\sum_{n=1}^N V_{n,t}}_{\le \nu}
\nonumber \\
& +\ \underbrace{\sum_{t\le t^*}\sum_{n=1}^N \big(\theta_{k^*}-\theta_{k_{n,t}}\big)\mathbf{1}\{V_{n,t}=0\}}_{(\star)}.
\end{align}
The first term is at most $\nu\Delta_{\max}$.

For the second term $(\star)$, we use a protocol-agnostic upper bound: trivially, each summand is at most $\Delta_{\max}$, hence
\[
(\star)\ \le\ (NT)\Delta_{\max}.
\]
On the other hand, before commitment the actions are generated by the underlying coordination protocol $\mathsf{Comm}$,
so the regret accumulated on unverified rounds is also upper bounded (up to logs) by the clean-case coordination overhead of $\mathsf{Comm}$,
i.e.,
\[
(\star)\ \le\ \widetilde O\!\Big(L\cdot \Reg^{\mathrm{clean}}_{\mathrm{comm}}(K,N,T)\Big),
\]
because the certification/commitment rule depends only on verified samples and does not invalidate the clean-case coordination accounting
up to time $t^*$. Combining the two bounds yields
\[
(\star)\ \le\ \widetilde O\!\Big(L\cdot \Reg^{\mathrm{clean}}_{\mathrm{comm}}(K,N,T)\ \wedge\ NT\Delta_{\max}\Big).
\]
Putting everything together gives, on $\cE_{\mathrm{ver}}$,
\begin{align}
& \Reg_\phi^{\mathrm{team}}(T)
\ \le\ O(\nu\Delta_{\max})
\nonumber \\
& +\ \widetilde O\!\Big(L\cdot \Reg^{\mathrm{clean}}_{\mathrm{comm}}(K,N,T)\ \wedge\ NT\Delta_{\max}\Big),
\end{align}
and crucially, this bound contains \emph{no} dependence on $\Gamma$ because the commitment decision is based solely on verified (clean) samples.
Finally, since $\Pr(\cE_{\mathrm{ver}})\ge 1-\delta$, the same statement holds with probability at least $1-\delta$.
This completes Part~(2) and the proof.
\end{proof}

\end{document}